\documentclass{article}
\pdfoutput=1
\usepackage{hyperref}
\usepackage{microtype}
\usepackage{caption}
\usepackage{pgfplots}
\pgfplotsset{compat=1.9}

\usepackage{multirow}
\usepackage{xspace}
\usepackage{amsmath}
\usepackage{amssymb}
\usepackage{amsthm,enumitem,dsfont}
\usepackage{algorithm}
\usepackage[noend]{algpseudocode}
\usepackage{subfloat}
\usepackage{subfig}
\usepackage{float}
\usepackage{graphicx}
\usepackage{soul}

\usepackage[accepted]{icml2020}

\newcommand\sN{\ensuremath{\mathcal{N}}}

\newcommand\sS{\ensuremath{\mathcal{S}}}

\newcommand\R{\ensuremath{\mathbb{R}}} 
\newcommand\eqdef{\ensuremath{\stackrel{\rm def}{=}}} 

\ifthenelse{\isundefined{\definition}}{\newtheorem{definition}{Definition}}{}
\ifthenelse{\isundefined{\assumption}}{\newtheorem{assumption}{Assumption}}{}
\ifthenelse{\isundefined{\hypothesis}}{\newtheorem{hypothesis}{Hypothesis}}{}
\ifthenelse{\isundefined{\proposition}}{\newtheorem{proposition}{Proposition}}{}
\ifthenelse{\isundefined{\theorem}}{\newtheorem{theorem}{Theorem}}{}
\ifthenelse{\isundefined{\lemma}}{\newtheorem{lemma}{Lemma}}{}
\ifthenelse{\isundefined{\corollary}}{\newtheorem{corollary}{Corollary}}{}
\ifthenelse{\isundefined{\alg}}{\newtheorem{alg}{Algorithm}}{}
\ifthenelse{\isundefined{\example}}{\newtheorem{example}{Example}}{}

\begin{document}

\newcommand{\ar}[1]{{\bf \color{blue} Aditi: #1}}
\newcommand{\pj}[1]{{\bf \color{red} Prateek: #1}}
\newcommand{\pset}{\sS}
\newcommand{\negset}[1]{N_{#1}}
\newcommand{\negsetlarge}[1]{\bar{N}_{#1}}
\newcommand{\thetaoc}{\theta^\text{svdd}}
\newcommand{\thetadrocc}{\theta^\text{drocc}}
\newcommand{\thetalf}{\theta^\text{lf}}
\newcommand{\elldrocc}{\ell^\text{dr}}
\newcommand{\ngrad}{m}
\newcommand{\ninit}{n_0}

\newcommand{\drocc}{\ensuremath{{\rm \text{DROCC}}}\xspace}
\newcommand{\droccrand}{\ensuremath{{\rm \text{DROCC--Rand}}}\xspace}
\newcommand{\lfoc}{\ensuremath{{\rm OCLN}}\xspace}
\newcommand{\muoc}{\mu^\text{oc}}
\newcommand{\roc}{R^\text{oc}}
\newcommand{\mudrocc}{\mu^\text{drooc}}
\newcommand{\rdrocc}{R^\text{drooc}}
\newcommand{\drocclf}{\ensuremath{{\rm \text{DROCC--LF}}}\xspace}
\newcommand{\droccoe}{\ensuremath{{\rm \text{DROCC--OE}}}\xspace}

\twocolumn[
\icmltitlerunning{DROCC: Deep Robust One-Class Classification}

\icmltitle{DROCC: Deep Robust One-Class Classification}



\icmlsetsymbol{note}{*}
\begin{icmlauthorlist}
\icmlauthor{Sachin Goyal}{msr}
\icmlauthor{Aditi Raghunathan}{stanford,note}
\icmlauthor{Moksh Jain}{nitk,note}
\icmlauthor{Harsha Simhadri}{msr}
\icmlauthor{Prateek Jain}{msr}
\end{icmlauthorlist}

\icmlaffiliation{nitk}{NITK Surathkal}
\icmlaffiliation{stanford}{Stanford University}
\icmlaffiliation{msr}{Microsoft Research India}

\icmlcorrespondingauthor{Prateek Jain}{prajain@microsoft.com}

\icmlkeywords{Machine Learning, ICML}

\vskip 0.3in
]



\printAffiliationsAndNotice{\textsuperscript{*}Part of the work was done while interning at Microsoft Research India.}  

\begin{abstract}
 Classical approaches for one-class problems such as one-class SVM and isolation forest require careful feature engineering when applied to structured domains like images. State-of-the-art methods aim to leverage deep learning to learn appropriate features via two main approaches. The first approach based on predicting transformations~\cite{golan2018,hendrycks2018deep} while successful in some domains, crucially depends on an appropriate domain-specific set of transformations that are hard to obtain in general. The second approach of minimizing a classical one-class loss on the \emph{learned} final layer representations, e.g., DeepSVDD \cite{deepsvdd} suffers from the fundamental drawback of representation collapse. In this work, we propose Deep Robust One Class Classification (DROCC) that is both applicable to most standard domains \emph{without} requiring any side-information and robust to representation collapse. DROCC is based on the assumption that the points from the class of interest lie on a well-sampled, locally linear low dimensional manifold. Empirical evaluation demonstrates that DROCC is highly effective in two different one-class problem settings and on a range of real-world datasets across different domains: tabular data, images (CIFAR and ImageNet), audio, and time-series, offering up to 20\% increase in accuracy over the state-of-the-art in anomaly detection. Code is available at \url{https://github.com/microsoft/EdgeML}.
\end{abstract}  

\section{Introduction}
\label{sec:intro}
In this work, we study ``one-class" classification where the goal is to obtain accurate discriminators for a special class. Anomaly detection is one of the most well-known problems in this setting where we want to identify outliers, i.e. points that do not belong to the typical data (special class). Another related setting under this framework is classification from limited negative training instances where we require low false positive rate at test time even over close negatives. This is common in AI systems such as wake-word\footnote{audio or visual cue that triggers some action from the system} detection where the wake-words form the positive or special class, and for safe operation in the real world, the system should not fire on inputs that are close but not identical to the wake-words, no matter how the training data was sampled.



Anomaly detection is a well-studied problem with a large body of research \cite{aggarwal16,chandola2009anomaly}. Classical approaches for anomaly detection are based on modeling the typical data using simple functions over the inputs \cite{oneclasssvm,isolationforest,lakhina2004}, such as constructing a minimum-enclosing ball around the typical data points \cite{tax2004support}. While these techniques are well-suited when the input is featurized appropriately, they struggle on complex domains like vision and speech, where hand-designing features is difficult. 

In contrast, deep learning based anomaly detection methods attempt to automatically \emph{learn} features, e.g., using CNNs in vision \cite{deepsvdd}. However, current approaches to do so have fundamental limitations. One family of approaches is based on extending the classical data modeling techniques over the learned representations. However, learning these representations jointly with the data modeling layer might lead to degenerate solutions where all the points are mapped to a single point (like origin), and the data modeling layer can now perfectly ``fit'' the typical data. Recent works like \cite{deepsvdd} have proposed some heuristics to mitigate this like setting the bias to zero, but such heuristics are often insufficient in practice (Table~\ref{tab:CIFAR}). The second line of work \citep{golan2018, bergman2020classificationbased, hendryks2019} is based on learning the salient geometric structure of the typical data (e.g., orientation of the object) by applying specific transformations (e.g., rotations and flips) to the input data and training the discriminator to predict applied transformation. If the discriminator fails to predict the transform accurately, the input does not have the same orientation as typical data and is considered anomalous. In order to be successful, these works critically rely on side-information in the form of appropriate structure/transformations, which is difficult to define in general, especially for domains like time-series, speech, etc. Even for images, if the normal data has been captured from multiple orientations, it is difficult to find appropriate transformations. The last set of deep anomaly detection techniques use generative models such as autoencoders or generative-adversarial networks (GANs) \cite{schlegl2017unsupervised} to learn to generate the entire typical data distribution which can be challenging and inaccurate in practice (Table~\ref{tab:CIFAR}). 

In this paper,  we propose a novel {\em Deep Robust One-Class Classifiation} (\drocc) method for anomaly detection that attempts to address the drawbacks of previous methods detailed above. \drocc is robust to representation collapse by involving a discriminative component that is general and is empirically accurate on most standard domains like tabular, time-series and vision without requiring any additional side-information. \drocc is motivated by the key observation that generally, the typical data lies on a low-dimensional manifold, which is well-sampled in the training data. This is believed to be true even in complex domains such as vision, speech, and natural language \cite{visionmanifold}. As manifolds resemble Euclidean space locally, our discriminative component is based on classifying a point as anomalous if it is \emph{outside} the union of small $\ell_2$ balls around the training typical points (See Figure~\ref{fig:synthetic_a} for an illustration). Importantly, the above definition allows us to synthetically generate anomalous points, and we adaptively generate the most effective anomalous points while training via a gradient ascent phase reminiscent of adversarial training. In other words, \drocc has a gradient ascent phase to adaptively add anomalous points to our training set and a gradient descent phase to minimize the classification loss by learning a representation and a classifier on top of the representations to separate typical points from the generated anomalous points. In this way, \drocc automatically learns an appropriate representation (like DeepSVDD)  but is robust to a representation collapse as mapping all points to the same value would lead to poor discrimination between normal points and the generated anomalous points. 
  
Next, we study a critical problem similar in flavor to anomaly detection and outlier exposure \cite{hendrycks2018deep}, which we refer to as One-class Classification with Limited Negatives (\lfoc). The goal of \lfoc is to design a one-class classifier for a {\em positive} class with only limited negative instances---the space of negatives is huge and is not well-sampled by the training points. The \lfoc classifier should have low FPR against \emph{arbitrary} distribution of negatives (or uninteresting class) while still ensuring accurate prediction accuracy for positives. For example, consider audio wake-word detection, where the goal is to identify a certain word, say \texttt{Marvin} in a given speech stream. For training, we collect negative instances where \texttt{Marvin} has not been uttered. Standard classification methods tend to identify simple patterns for classification, often relying only on some substring of \texttt{Marvin} say \texttt{Mar}. While such a classifier has good accuracy on the training set, in practice, it can have a high FPR as the classifier will mis-fire on utterances like \texttt{Marvelous} or \texttt{Martha}. This exact setting has been relatively less well-studied, and there is no benchmark to evaluate methods. Existing work suggests to simply expand the training data to include false positives found \emph{after} the model is deployed, which is expensive and oftentimes infeasible or unsafe in real applications. 

In contrast, we propose \drocclf, an outlier-exposure style extension of \drocc. Intuitively, \drocclf combines \drocc's anomaly detection loss (that is over only the positive data points) with standard classification loss over the negative data. But, in addition, \drocclf exploits the negative examples to learn a Mahalanobis distance to compare points over the manifold instead of using the standard Euclidean distance, which can be inaccurate for high-dimensional data with relatively fewer samples. 

We apply \drocc to standard benchmarks from multiple domains such as vision, audio, time-series, tabular data, and empirically observe that \drocc is indeed successful at modeling the positive (typical) class across all the above mentioned domains and can significantly outperform baselines. For example, when applied to the anomaly detection task on the benchmark CIFAR-10 dataset, our method can be up to 20\% more accurate than the baselines like DeepSVDD \cite{deepsvdd},  Autoencoder \cite{mayu2014}, and GAN based methods \cite{nguyen19b}. Similarly, for tabular data benchmarks like  Arrhythmia, \drocc can be $\ge 18$\% more accurate than state-of-the-art methods \cite{bergman2020classificationbased, zong2018deep}. Finally, for \lfoc problem, our method can be upto 10\% more accurate than standard baselines. 

In summary, the paper makes the following contributions: 
\begin{itemize}
\item We propose \drocc method that is based on a low-dimensional manifold assumption on the positive class using which it synthetically and adaptively generates negative instances to provide a general and robust approach to anomaly detection. 
\item We extend \drocc to a one-class classification problem where low FPR on arbitrary negatives is crucial. We also provide an experimental setup to evaluate different methods for this important but relatively less studied problem. 
\item Finally, we experiment with \drocc on a wide range of datasets across different domains--image, audio, time-series data and demonstrate the effectiveness of our method compared to baselines.
\end{itemize}

\section{Related Work}
\label{sec:rw}
Anomaly Detection (AD) has been extensively studied owing to its wide applicabilty \cite{chandola2009anomaly, goldstein2016a,aggarwal16}. Classical techniques use simple functions like modeling normal points via low-dimensional subspace or a tree-structured partition of the input space to detect anomalies \cite{oneclasssvm, tax2004support, isolationforest, lakhina2004, gu2019}. In contrast, deep AD methods attempt to learn appropriate features, while also learning how to model the typical data points using these features. They broadly fall into three categories discussed below. \\

{\bf AD via generative modeling.} Deep Autoencoders as well as GAN based methods have been studied extensively \cite{malhotra2016lstm, mayu2014, nguyen19b, gan-ad,anogan}. However, these methods solve a harder problem as they require reconstructing the {\em entire} input from its low-dimensional representation during the decoding step. In contrast, \drocc directly addresses the goal of only identifying if a given point lies {\em somewhere} on the manifold, and hence tends to be more accurate in practice (see Table~\ref{tab:CIFAR}, \ref{tab:tabular}, \ref{tab:timeseries}). \\

{\bf Deep Once Class SVM}: Deep SVDD \cite{deepsvdd} introduced the first deep one-class classification objective for anomaly detection, but suffers from  representation collapse issue (see Section~\ref{sec:intro}). In contrast, DROCC is robust to such collapse since the training objective requires representations to allow for accurate discrimination between typical data points and their perturbations that are off the manifold of the typical data points.\\

{\bf Transformations based methods}: Recently, \cite{golan2018, hendryks2019} proposed another approach to AD based on self-supervision. The training procedure involves applying different transformations to the typical points and training a classifier to identify the transform applied. The key assumption is that a point is normal iff the transformations applied to the point can be correctly identified, i.e., normal points conform to a specific structure captured by the transformations.  \cite{golan2018,hendryks2019} applied the method to vision datasets and proposed using rotations, flips etc as the transformations. \cite{bergman2020classificationbased} generalized the method to tabular data by using handcrafted affine transforms. Naturally, the transformations required by these methods are heavily domain dependent and are hard to design for domains like time-series. Furthermore, even for vision tasks, the suitability of a transformation varies based on the structure of the typical points. For example, as discussed in \cite{golan2018},  horizontal flips perform well when the typical points are from class '3' (AUROC 0.957) of MNIST but perform poorly when typical points are from class '8' (AUROC 0.646). In contrast, the low-dimensional manifold assumption that motivates \drocc is generic and seems to hold across several domains like images, speech, etc. For example, DROCC obtains AUROC of $\sim$ 0.97 on both typical class '8' and typical class '3' in MNIST. (See Section~\ref{sec:exp} for more comparison with self-supervision based techniques) \\

{\bf Side-information based AD}: Recently, several AD  methods to explicitly incorporate side-information have been proposed. \cite{hendrycks2018deep} leverages access to a few out-of-distribution samples, \citep{ruff2020deep} explores the semisupervised setting where a small set of labeled anomalous examples are available. We view these approaches as complementary to \drocc which does not assume any side-information. Finally, \lfoc problem is generally modeled as a binary classification probelm, but outlier exposure (OE) style formulation \cite{hendryks2019} can be used to combine it with anomaly detection methods. Our method \drocclf builds upon OE approach but exploits the "outliers" in a more integrated manner. 

\section{Anomaly Detection}
\label{sec:ad}
Let $\pset \subseteq \R^d$ denote the set of {\em typical}, i.e., non-anomalous data points. A point $x \in \R^d$ is \emph{anomalous} or {\em atypical} if $x \not \in \pset$. Suppose we are given $n$ samples $D=[x_i]_{i=1}^n \in \R^d$ as training data, where $D_S=\{x_i \mid x_i \in \pset\}$ is the set of typical points sampled in the training data and $|D_S|\geq (1-\nu)|S|$ i.e. $\nu\ll 1$  fraction of points in $D$ are anomalies. Then, the goal in \textit{unsupervised} anomaly detection (AD) is to learn a function $f_\theta: \R^d \mapsto \{-1, 1 \}$ such $f_\theta(x) = 1 $ when $x \in \pset$ and $f_\theta(x) = -1$ when $x \not \in \pset$. The anomaly detector is parameterized by some parameters $\theta$.

\begin{figure}
    \begin{minipage}{.49\textwidth}
        \begin{algorithm}[H]
            \algrenewcommand\algorithmicprocedure{}
            \caption{Training neural networks via \drocc} \label{alg:drocc-nn} \textbf{Input:} Training data $D=[x_1, x_2, \hdots x_n]$.  
            
            \textbf{Parameters:} Radius $r$, $\lambda\geq 0$, $\mu\geq 0$, step-size $\eta$, number of gradient steps $\ngrad$, number of initial training steps $\ninit$.
            
            \textbf{Initial steps:} For $B = 1, \hdots \ninit$
            
            \hspace*{\algorithmicindent} $X_B$: Batch of training inputs
            
            \hspace*{\algorithmicindent} $\theta = \theta - \text{Gradient-Step}\Big(\sum \limits_{x \in X_B} \ell(f_\theta(x), 1) \Big)$
            
            \textbf{DROCC steps:} For $B = \ninit, \hdots \ninit + N$ 
            
            \hspace*{\algorithmicindent} $X_B$: Batch of training inputs
            
            
            \hspace*{\algorithmicindent} $\forall x \in X_B: h \sim \sN(0, I_{d})$

            \hspace*{\algorithmicindent} \textbf{Adversarial search:} For $i = 1, \hdots \ngrad$
            
            \hspace*{\algorithmicindent} \hspace*{\algorithmicindent} 1. $\ell(h) = \ell(f_\theta(x + h), -1)$
            
            \hspace*{\algorithmicindent} \hspace*{\algorithmicindent} 2. $h = h + \eta \frac{\nabla_h \ell(h)}{\| \nabla_h \ell(h) \|}$
            
            \hspace*{\algorithmicindent} \hspace*{\algorithmicindent} 3. $h = \frac{\alpha}{\|h\|}\cdot h$ where $\alpha= r\cdot \mathds{1}[\|h\|\leq r]+ \|h\|\cdot \mathds{1}[r\leq \hspace*{60pt} \|h\|\leq \gamma \cdot r] + \gamma \cdot r\cdot \mathds{1}[\|h\|\geq \gamma \cdot r]$
            
            \hspace*{\algorithmicindent} $\ell^{itr} =  \lambda\| \theta \|^2 + \sum \limits_{x \in X_B} \ell(f_\theta(x), 1) +\mu  \ell(f_\theta(x + h), -1)  $

            \hspace*{\algorithmicindent} $\theta = \theta - \text{Gradient-Step}(\ell^{itr})$  
        \end{algorithm}
    \end{minipage}\hspace*{5pt}
    \begin{minipage}{.49\textwidth}

    \end{minipage}
\end{figure}

\begin{figure*}[t]%
    \centering
    \subfloat[\label{fig:synthetic_a}]{{\includegraphics[width=3.4cm]{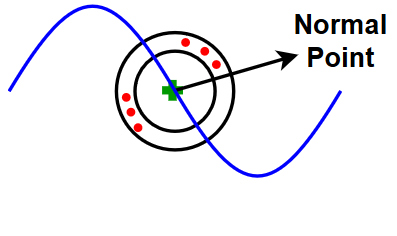}}}\qquad
    \subfloat[\label{fig:synthetic_b}]{{\includegraphics[width=3.4cm]{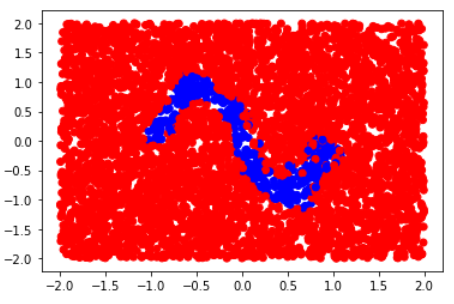}}}\qquad
    \subfloat[\label{fig:synthetic_c}]{{\includegraphics[width=3.4cm]{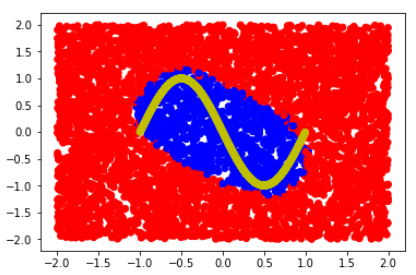}}}\qquad
    \subfloat[\label{fig:synthetic_d}]{{\includegraphics[width=3.1cm]{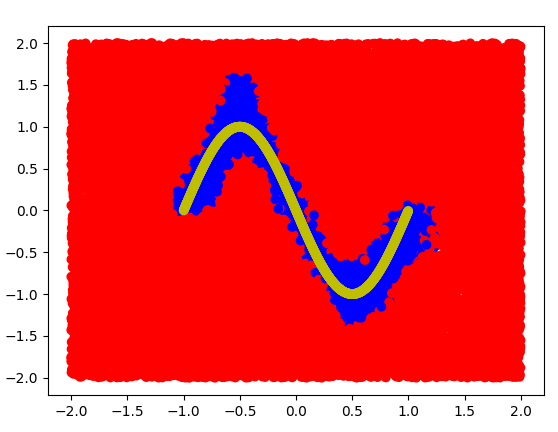}}}
    \caption{(a) A {\em normal} data manifold with red dots representing generated anomalous points in $\negset{i}(r)$. (b) Decision boundary learned by \drocc when applied to the data from (a). Blue represents points classified as normal and red points are classified as abnormal. (c), (d):  first two dimensions of the decision boundary of \drocc  and \drocclf, when applied to noisy  data (Section~\ref{sec:exp_fpr}). \drocclf is nearly optimal while \drocc's decision boundary is inaccurate. Yellow color sine wave depicts the  train data.}
\end{figure*}
{\bf Deep Robust One Class Classification}: 
\label{sec:drocc_formulation}
We now present our approach to unsupervised anomaly detection that we call Deep Robust One Class Classification (\drocc). Our approach is based on the following hypothesis:
\emph{The set of typical points $\pset$ lies on a low dimensional locally linear manifold that is well-sampled}. In other words, outside a small radius around a training (typical) point, most points are anomalous. Furthermore, as manifolds are locally Euclidean, we can use the standard $\ell_2$ distance function to compare the points in a small neighborhood. Figure \ref{fig:synthetic_a} shows a $1$-d manifold of the typical points and intuitively, why in a small neighborhood of the training points we can use $\ell_2$ distances. We label the typical points as positive and anomalous points as negative. 
  
Formally, for a DNN architecture $f_\theta:\R^d \rightarrow \R$ parameterized by $\theta$, and a small radius $r>0$, $\drocc$ estimates parameter $\theta^{dr}$ as : $\min_\theta \ell^{dr}(\theta)$, where,

\begin{align}
    &\elldrocc(\theta) = \lambda\| \theta \|^2 +  \sum\limits_{i=1}^n [\ell(f_\theta(x_i), 1) + \mu \max \limits_{\stackrel{\tilde{x}_i\in}{ \negset{i}(r)}} \ell(f_\theta(\tilde{x}_i), -1)],\ \nonumber\\
    &\negset{i}(r)\eqdef \Big\{\| \tilde{x}_i - x_i \|_2 \leq \gamma \cdot r;\ \ r\leq \| \tilde{x}_i - x_j \|, \
    \nonumber\\
    &\qquad\qquad~\forall j = 1, 2, \hdots n\Big\}, \label{eq:drocc-loss}
\end{align}

and $\lambda>0$, $\mu>0$ are regularization parameters. $\negset{i}(r)$ captures points off the manifold, i.e., are at least at $r$ distance from all training points. We use an upper bound $\gamma \cdot r$ for regularizing the optimization problem where $\gamma\geq 1$. $\ell:\mathbb{R}\times \mathbb{R} \rightarrow \mathbb{R}$ is some classification loss function, and goal is to classify the given normal points $x_i$'s as positives while the generated anomalous examples $\tilde{x}_i$ as negatives. 

The above given formulation is a saddle point problem and is similar to adversarial training where the network is trained to be robust to worst-case $\ell_p$ ball perturbations around the inputs (See, for example \citep{madry2018towards}). In \drocc, we replace the $\ell_p$ ball with $\negset{i}(r)$, and adopt the standard projected gradient descent-ascent technique to solve the saddle point problem.

\paragraph{Gradient-ascent to generate negatives.} A key step in the gradient descent-ascent algorithm is that of projection onto the $\negset{i}(r)$ set. That is, given $z\in \R^d$, the goal is to find $\tilde{x}_i = \arg\min_u \|u-z\|^2\ s.t.\ u\in \negset{i}(r)$. Now, $\negset{i}$ contains points that are less than $\gamma \cdot r$ distance away from $x_i$ and at least $r$ away from all $x_j$'s. The second constraint involves all the training points and is computationally challenging. So, for computational ease, we redefine $\negset{i}(r)\eqdef \Big\{r\leq \| \tilde{x}_i - x_i \|_2 \leq \gamma \cdot r\Big\}$. In practice, since the positive  points in $\pset$ lie on a low dimensional manifold, we empirically find that the adversarial search over this set does not yield a point that is in $\pset$. Further, we use a lower weight on the classification loss of the generated negatives so as to guard against possible non-anomalous points in $\negset{i}(r)$. Finally, projection onto this set is given by $\tilde{x}_i= x_i+\alpha\cdot (z-x_i)$ where $\beta=\|z-x_i\|$, and $\alpha=\gamma r/\beta$ if $\beta\geq \gamma r$ (point is too far), $\alpha=r/\beta$ if $\beta\leq r$ and $\alpha=1$ otherwise. 

Algorithm~\ref{alg:drocc-nn} summarizes our \drocc method. The three steps in the adversarial search are performed in parallel for each $x\in B$ the batch; for simplicity, we present the procedure for a single example $x$. In step one, we compute the loss of the network with respect to a negative label (anomalous point) where we express $\tilde{x}$ as $x + h$. In step two, we maximize this loss in order to find the most ``adversarial'' point via normalized steepest ascent. Finally, we project $\tilde{x}$ onto $\negset{i}(r)$. 
In order to update the parameters of the network, we could use any gradient based update rule such as SGD or other adaptive methods like Adagrad or Adam. We typically set $\gamma=2$. Parameters $\lambda, \mu, \eta$ are selected via cross-validation. Note that our method allows arbitrary DNN architecture $f_\theta$ to represent and classify data points $x_i$. Finally, we set $\ell$ to be the standard cross-entropy loss. 

\section{One-class Classification with Limited Negatives (\lfoc)}
\label{sec:fpr}
In this section, we extend \drocc to address the \lfoc problem. Let $D=[(x_1, y_1), \dots, (x_n,y_n)]$ be a given set of points where $x_i\in \R^d$ and $y_i \in \{1,-1\}$. Furthermore, let the mass of positive points' distribution covered by the training data is significantly higher than that of negative points' distribution. For example, if data points are sampled from a discrete distribution, with $P_+$ being the marginal distribution of the positive points and $P_-$ be the margin distribution of the negative points. Then, the assumption is: $\frac{1}{n_-}\sum_{i, y_i=-1} P_-(x_i) \ll \frac{1}{n_+}\sum_{i, y_i=1} P_+(x_i)$ where $n_+$, $n_-$ are the number of positive and negative training points. 

The goal of \lfoc is similar to anomaly detection (AD), that is, to identify arbitrary outliers--negative class in this case--correctly despite limited access to negatives' data distribution.  So it is an AD problem with side-information in the form of limited negatives. Intuitively, \lfoc problems arise in domains where data for special positive class (or set of classes) can be collected thoroughly, but the ``negative" class is a catch-all class that cannot be sampled thoroughly due to its sheer size. Several real-world problems can be naturally modeled by \lfoc. For example, consider wake word detection problems where the goal is to identify a special audio command to wake up the system. Here, the data for a special wake word can be collected exhaustively, but the negative class, which is ``everything else" cannot be sampled properly. 

Naturally, we can directly apply standard AD methods (e.g., \drocc) or binary classification methods to the problem. However, AD methods ignore the side-information, while the classification methods' might generalize only to the training distribution of negatives and hence might have high False Positive Rate (FPR) in the presence of negatives far from the train distribution. Instead, we propose method \droccoe that uses an approach similar to outlier exposure \cite{hendrycks2018deep}, where the optimization function is given by a summation of the anomaly detection loss and standard cross entropy loss over negatives. The intuition behind \droccoe is that the positive data satisfies the manifold hypothesis of the previous section, and hence points off the manifold should be classified as negatives. But the process can be bootstrapped by explicit negatives from the training data. 

Next, we propose \drocclf which integrates information from the negatives in a deeper manner than \droccoe. In particular, we use negatives to learn input coordinates or features which are noisy and should be ignored. As \drocc uses Euclidean distance to compare points locally, it might struggle due to the noisy coordinates, which \drocclf will be able to ignore. Formally, \drocclf estimates parameter $\thetalf$ as: $\min_\theta \ell^{lf}(\theta)$ where, 
{
    \begin{align}
    &\ell^{lf}(\theta) = \lambda\| \theta \|^2 +  \sum\limits_{i=1}^n [\ell(f_\theta(x_i), y_i) + \mu \max \limits_{\stackrel{\tilde{x}_i\in}{ \negset{i}(r)}} \ell(f_\theta(\tilde{x}_i), -1)],\ 
    \nonumber \\
    &\qquad\qquad\negset{i}(r):=\{\tilde{x}_i,\ s.t., \ r\leq \|\tilde{x}_i-x_i\|_\Sigma \leq \gamma\cdot r\},\label{eq:drocclf-loss}
    \end{align}
}
and $\lambda>0$, $\mu>0$ are regularization parameters. Instead of Euclidean distance, we use Mahalanobis distance function $\|\tilde{x}-x\|_\Sigma^2=\sum_j \sigma_j (\tilde{x}^j-x^j)^2$ where $x^j, \tilde{x}^j$ are the $j$-th coordinate of $x$ and $\tilde{x}$, respectively.  $\sigma_j:= \left|\frac{\partial f_{\theta(x)}}{\partial x_j}\right|$, i.e., $\sigma_j$ measures the "influence" of $j$-th coordinate on the output, and is updated every epoch during training.

Similar to \eqref{eq:drocc-loss}, we can use the standard projected gradient descent-ascent algorithm to optimize the above given saddle point problem. Here again, projection onto $\negset{i}(r)$ is the key step. That is, the goal is to find: $\tilde{x}_i = \arg\min_x \|x-z\|^2\ s.t.\ x\in \negset{i}(r)$. Unlike, Section~\ref{sec:ad} and Algorithm~\ref{alg:drocc-nn},  the above projection is unlikely to be available in closed form and requires more careful arguments. 
\begin{proposition}\label{prop:lfoc}
Consider the problem: $\min_{\tilde{x}} \|\tilde{x}-z\|^2,\ s.t.,\  r^2\leq \|\tilde{x}-x\|_\Sigma^2 \leq \gamma^2 r^2 $ and let $\delta=z-x$. If $r\leq \|\delta\|_\Sigma \leq \gamma r$, then $\tilde{x}=z$. Otherwise, the optimal solution is : $\tilde{x}=x+(I+\tau \Sigma)^{-1}\delta,$ where : \\
1) If $\|\delta\|_\Sigma\leq r$, \\
$\tau :=\arg\min_{\tau\leq 0} \sum_j \frac{\delta_j^2 \tau^2 \sigma_j^2}{(1+\tau \sigma_j)^2},\ s.t.,\ \sum_j \frac{\delta_j^2\sigma_j}{(1+\tau \sigma_j)^2}\geq r^2,$  \\
2) If   $\|\delta\|_\Sigma\geq \gamma\cdot r$, \\
$\tau^{-1}:=\arg\min_{\nu\geq 0} \sum_j \frac{\delta_j^2 \sigma_j^2}{(\nu+ \sigma_j)^2},\ s.t.,\ \sum_j \frac{\delta_j^2\sigma_j\nu^2}{(\nu+ \sigma_j)^2}\leq \gamma^2 r^2$ .
\end{proposition}
See Appendix~\ref{app:lfoc} for a detailed proof. The above proposition reduces the projection problem to a non-convex but one-dimensional optimization problem. We solve this problem via standard grid search over: $\tau=[-\frac{1}{\max_j \sigma_j}, 0]$ or $\nu=[0,\frac{\alpha}{1-\alpha}\max_j \sigma_j ]$ where $\alpha=\gamma\cdot r/\|\delta\|_\Sigma$. The algorithm is now almost same as Algorithm~\ref{alg:drocc-nn} but uses the
above mentioned projection algorithm; see Appendix~\ref{app:lfoc} for a pseudo-code of our \drocclf method.

\subsection{\lfoc Evaluation Setup}
Due to lack of benchmarks, it is difficult to evaluate a solution for \lfoc. So, we provide a novel experimental setup for a wake-word detection and a digit classification problem, showing that \drocclf indeed significantly outperforms standard anomaly detection, binary classification, and \droccoe on practically relevant metrics (Section~\ref{sec:exp_fpr}). 

In particular, our setup is inspired by standard settings encountered by real-world detection problems. For example, consider the wakeword detection problem, where the goal is to detect a wakeword like say "Marvin" in a continuous stream of data. In this setting, we are provided a few positive examples for {\em Marvin} and a few generic negative examples from everyday speech. But, in our experiment setup, we generate {\em close} or difficult negatives by generating examples like {\em Arvin}, {\em Marvelous}  etc. Now, in most real-world deployments, a critical requirement is low False Positive Rates, even on such difficult negatives. So, we study various methods with FPR bounded by say 3\% or 5\% on negative data that comprises of generic negatives as well as difficult {\em close} negatives. Now, under FPR constraint, we evaluate various methods by their recall rate, i.e., based on how many true positives the method is able to identify. We propose a similar setup for a digit classification problem as well; see Section~\ref{sec:exp_fpr} for more details. 

\section{Empirical Evaluation}
\label{sec:exp}
In this section, we present empirical evaluation of \drocc on two one-class classification problems: Anomaly Detection and One-Class Classification with Limited Negatives (\lfoc). We discuss the experimental setup, datasets, baselines, and the implementation details. Through experimental results on a wide range of synthetic and real-world datasets, we present strong empirical evidence for the effectiveness of our approach for one-class classification problems.
\addtolength{\dbltextfloatsep}{-0.06in}
\addtolength{\textfloatsep}{-0.06in}
\begin{table*}[t]
    \centering
    \caption{Average AUC (with standard deviation) for one-vs-all anomaly detection on CIFAR-10. \drocc outperforms baselines on most classes, with gains as high at 20\%, and notably, nearest neighbours beats all the baselines on 2 classes.}
    \resizebox{\textwidth}{!}{
        \begin{tabular}{llllllllll}
            \hline
            \textbf{CIFAR Class} & \textbf{OC-SVM} & \textbf{IF} & \textbf{DCAE} & \textbf{AnoGAN} & 
            \textbf{\begin{tabular}[c]{@{}l@{}}ConAD\\ 16\end{tabular}} & \textbf{\begin{tabular}[c]{@{}l@{}}Soft-Bound\\ Deep SVDD\end{tabular}} & \textbf{\begin{tabular}[c]{@{}l@{}}One-Class \\ Deep SVDD\end{tabular}} & \textbf{\begin{tabular}[c]{@{}l@{}}Nearest\\ Neighbour\end{tabular}} & \textbf{\drocc (Ours)}  \\ \hline
            Airplane             & 61.6$\pm$0.9     & 60.1$\pm$0.7    & 59\_1$\pm$5\_1    & 67.1$\pm$2.5        & 77.2 & 61.7$\pm$4.2                                                                & 61.7$\pm$4.1                                                                & 69.02                      & \textbf{81.66 $\pm$ 0.22}  \\
            Automobile           & 63.8$\pm$0.6      & 50.8$\pm$0.6    & 57.4$\pm$2.9      & 54.7$\pm$3.4        & 63.1 & 64.8$\pm$1.4                                                                & 65.9$\pm$2.1                                                                & 44.2                       & \textbf{76.738 $\pm$ 0.99} \\
            Bird                 & 50.0$\pm$0.5      & 49.2$\pm$0.4    & 48.9$\pm$2.4      & 52.9$\pm$3.0        & 63.1 & 49.5$\pm$1.4                                                                & 50.8$\pm$0.8                                                                & \textbf{68.27}             & 66.664 $\pm$ 0.96          \\
            Cat                  & 55.9$\pm$1.3      & 55.1$\pm$0.4    & 58.4$\pm$1.2      & 54.5$\pm$1.9        & 61.5 & 56.0$\pm$1.1                                                                & 59.1$\pm$1.4                                                                & 51.32                      & \textbf{67.132 $\pm$ 1.51} \\
            Deer                 & 66.0$\pm$0.7      & 49.8$\pm$0.4    & 54.0$\pm$1.3      & 65.1$\pm$3.2        & 63.3 & 59.1$\pm$1.1                                                                & 60.9$\pm$1.1                                                                & \textbf{76.71}             & 73.624 $\pm$ 2.00          \\
            Dog                  & 62.4$\pm$0.8      & 58.5$\pm$0.4    & 62.2$\pm$1.8      & 60.3$\pm$2.6        & 58.8 & 62.1$\pm$2.4                                                                & 65.7$\pm$2.5                                                                & 49.97                      & \textbf{74.434 $\pm$ 1.95} \\
            Frog                 & 74.7$\pm$0.3      & 42.9$\pm$0.6    & 51.2$\pm$5.2      & 58.5$\pm$1.4        & 69.1 & 67.8$\pm$2.4                                                                & 67.7$\pm$2.6                                                                & 72.44                      & \textbf{74.426 $\pm$ 0.92} \\
            Horse                & 62.6$\pm$0.6      & 55.1$\pm$0.7    & 58.6$\pm$2.9      & 62.5$\pm$0.8        & 64.0 & 65.2$\pm$1.0                                                                     & 67.3$\pm$0.9                                                                & 51.13                      & \textbf{71.39 $\pm$ 0.22}  \\
            Ship                 & 74.9$\pm$0.4      & 74.2$\pm$0.6    & 76.8$\pm$1.4      & 75.8$\pm$4.1        & 75.5 & 75.611.7                                                                & 75.9$\pm$1.2                                                                & 69.09                      & \textbf{80.016 $\pm$ 1.69} \\
            Truck                & 75.9$\pm$0.3      & 58.9$\pm$0.7    & 67.3$\pm$3.0      & 66.5$\pm$2.8        & 63.7 & 71.0$\pm$1.1                                                                & 73.1$\pm$1.2                                                                & 43.33                      & \textbf{76.21 $\pm$ 0.67} \\ \hline
    \end{tabular}}
    \label{tab:CIFAR}
\end{table*}
\subsection{Anomaly Detection}
\textbf{Datasets}: In all the experiments with multi-class datasets, we follow the standard one-vs-all setting for anomaly detection: fixing each class once as nominal and treating rest as anomaly. The model is trained only on the nominal class but the test data is sampled from all the classes. For timeseries datasets,  \textit{N} represents the number of time-steps/frames and \textit{d} represents the input feature length. 

We perform experiments on the following datasets:
\begin{itemize}[leftmargin=*]
    \itemsep 0pt
    \topsep 0pt
    \parskip 0pt
    \item 2-D sine-wave: 1000 points sampled uniformly from a 2-dimensional sine wave (see Figure~\ref{fig:synthetic_a}).
    \item Abalone \cite{ucirepo}: Physical measurements of abalone are provided and the task is to predict the age. Classes 3 and 21 are anomalies and classes 8, 9, and 10 are normal \cite{das2018active}.
    \item Arrthythmia \cite{odds}: Features derived from ECG and the task is to identify arrhythmic samples. Dimensionality is 279 but five categorical attributes are discarded. Dataset preparation is similar to \citet{zong2018deep}.
    \item Thyroid \cite{odds}: Determine whether a patient referred to the clinic is
hypothyroid based on patient's medical data. Only 6 continuous attributes are considered. Dataset preparation is same as \citet{zong2018deep}.
    \item Epileptic Seizure Recognition \cite{andrzejak2001indications}: EEG based time-series dataset from multiple patients. Task is to identify if EEG is abnormal (\textit{N} = 178, \textit{d} = 1).
    \item Audio Commands \cite{warden2018speech}: A multiclass data with 35 classes of audio keywords. Data is featurized using MFCC features with 32 filter banks over 25ms length windows with stride of 10ms ($N=98$, $d=32$). Dataset preparation is same as \citet{fastgrnn}.    
    \item CIFAR-10 \cite{krizhevsky09learningmultiple}: Widely used benchmark for anomaly detection, 10 classes with $32 \times 32$ images. 
    \item ImageNet-10: a subset of 10 randomly chosen classes from the ImageNet dataset \cite{deng2009imagenet} which contains $224\times 224$ color images.

\end{itemize}
The datasets which we use are all publicly available. We use the train-test splits when already available with a 80-20 split for train and validation set. In all other cases, we use random 60-20-20 split for train, validation, and test.

\textbf{\drocc Implementation:} The main hyper-parameter of our algorithm is the radius $r$ which defines the set $\negset{i}(r)$.
We observe that tweaking radius value around $\sqrt{d}/2$ (where \textit{d} is the dimension of the input data ) works the best, as due to zero-mean, unit-variance normalized features, the average distance between random points is $\approx\sqrt{d}$. We fix $\gamma$ as 2 in our experiments unless specified otherwise. Parameter $\mu$ \eqref{eq:drocc-loss}  is chosen from \{0.5 , 1.0\}. We use a standard step size from \{0.1, 0.01\} for gradient ascent and from $\{10^{-2}, 10^{-4}\}$ for gradient descent; we also tune the optimizer $\in$ \{Adam, SGD\}. See Appendix~\ref{app:ablation} for a detailed ablation study. The implementation is available as part of the EdgeML package~\cite{edgeml03}. The experiments were run on an Intel Xeon CPU with 12 cores clocked at 2.60 GHz and with NVIDIA Tesla P40 GPU, CUDA 10.2, and cuDNN 7.6. 

\subsubsection{Results}
\label{sec:exp_sine}
\textbf{Synthetic Data}: 
We present results on a simple 2-D sine wave dataset to visualize the kind of classifiers learnt by \drocc. Here, the positive data lies on a 1-D manifold given in Figure~\ref{fig:synthetic_a}. We observe from Figure~\ref{fig:synthetic_b} that \drocc is able to capture the manifold accurately; whereas the classical methods OC-SVM and DeepSVDD (shown in Appendix~\ref{app:model-vis}) perform poorly as they both try to learn a minimum enclosing ball for the \textit{whole} set of positive data points.

\begin{table}[t]
\centering
\caption{F1-Score (with standard deviation) for one-vs-all anomaly detection on Thyroid, Arrhythmia, and Abalone datasets. \drocc outperforms the baselines on all the three datasets.}
\resizebox{\columnwidth}{!}{
\begin{tabular}{llll}
\hline
\multicolumn{1}{c}{\textbf{}} & \multicolumn{3}{c}{\textbf{F1-Score}}                                                   \\ \cline{2-4} 
\textbf{Method}               & \textbf{Thyroid}            & \textbf{Arrhythmia}         & \textbf{Abalone}            \\ \hline
OC-SVM    \cite{oneclasssvm}                    & 0.39 $\pm$ 0.01          & 0.46 $\pm$ 0.00         & 0.48 $\pm$ 0.00         \\
DCN\cite{caron2018deep}                           & 0.33 $\pm$ 0.03          & 0.38 $\pm$ 0.03          & 0.40 $\pm$ 0.01         \\
E2E-AE  \cite{zong2018deep}                      & 0.13 $\pm$ 0.04          & 0.45 $\pm$ 0.03          & 0.33 $\pm$ 0.03          \\
LOF     \cite{breunig2000lof}                      & 0.54 $\pm$ 0.01          & 0.51 $\pm$ 0.01          & 0.33 $\pm$ 0.01          \\
DAGMM \cite{zong2018deep}                        & 0.49 $\pm$ 0.04          & 0.49 $\pm$ 0.03          & 0.20 $\pm$ 0.03          \\
DeepSVDD \cite{deepsvdd}           & 0.73 $\pm$ 0.00          & 0.54 $\pm$ 0.01          & 0.62 $\pm$ 0.01          \\
GOAD \cite{bergman2020classificationbased}           & 0.72 $\pm$ 0.01          & 0.51 $\pm$ 0.02          & 0.61 $\pm$ 0.02          \\
\textbf{DROCC (Ours)}       & \textbf{0.78 $\pm$ 0.03} & \textbf{0.69 $\pm$ 0.02} & \textbf{0.68 $\pm$ 0.02} \\ \hline
\end{tabular}
}
\label{tab:tabular}
\vspace{5pt}
\end{table}

\textbf{Tabular Data}:
Table~\ref{tab:tabular} compares \drocc against various classical algorithms: OC-SVM, LOF\cite{breunig2000lof} as well as deep baselines: DCN\cite{caron2018deep}, Autoencoder \cite{zong2018deep}, DAGMM\cite{zong2018deep}, DeepSVDD and GOAD\cite{bergman2020classificationbased} on the widely used benchmark datasets, Arrhythmia, Thyroid and Abalone. In line with prior work, we use the F1- Score for comparing the methods \cite{bergman2020classificationbased, zong2018deep}. A fully-connected network with a single hidden layer is used as the base network for all the deep baselines. We observe significant gains across all the three datasets for \drocc, as high as 18\% in Arrhythmia.

\begin{table}[t]
\centering
\caption{AUC (with standard deviation) for one-vs-all anomaly detection on Epileptic Seizures and Audio Keyword ``Marvin". \drocc outperforms the baselines on both the datasets}
\resizebox{\columnwidth}{!}{
\begin{tabular}{lll}
    \hline
    \multicolumn{1}{c}{\multirow{2}{*}{\textbf{Method}}} & \multicolumn{2}{c}{\textbf{AUC}}                                       \\ \cline{2-3} 
    \multicolumn{1}{c}{}                                 & \textbf{Epileptic Seizure}   & \textbf{Audio Keywords} \\ \hline
    kNN                                    & 91.75       & 65.81                   \\
    AE                                  \cite{mayu2014}   & 91.15 $\pm$ 1.7            & 51.49 $\pm$ 1.9         \\
    REBM                     \cite{rebm}                            & 97.24 $\pm$ 2.1           & 63.73 $\pm$ 2.4         \\
    DeepSVDD  \cite{deepsvdd}                                            & 94.84 $\pm$ 1.7           & 68.38 $\pm$ 1.8         \\
    \textbf{DROCC (Ours)}                                                & \textbf{98.23 $\pm$ 0.7}            &  \textbf{70.21 $\pm$ 1.1}    \\ \hline
\end{tabular}
}
\label{tab:timeseries}
\vspace{10pt}
\end{table}

\textbf{Time-Series Data}: There is a lack of work on anomaly detection for time-series datasets. Hence we extensively evaluate our method \drocc against deep baselines like AutoEncoders \cite{mayu2014}, REBM \cite{rebm} and DeepSVDD. For autoencoders, we use the architecture presented in \citet{srivastava2015unsupervised}.
A single layer LSTM is used for all the deep baselines. Motivated by recent analysis \cite{gu2019}, we also include nearest neighbours as a baseline. Table~\ref{tab:timeseries} compares the performance of \drocc against these baselines on the univariate Epileptic Seizure dataset, and the Audio Commands dataset. \drocc outperforms the baselines on both the datasets.

\textbf{Image Data}: 
For experiments on image datasets, we fixed $\gamma$ as 1. Table~\ref{tab:CIFAR} compares \drocc on CIFAR-10 against baseline numbers from OC-SVM \cite{oneclasssvm}, IF \cite{isolationforest}, DCAE \cite{seebock2016identifying}, AnoGAN \cite{anogan}, DeepSVDD as reported by \citet{deepsvdd} and against ConvAD16 as reported by \citet{nguyen19b}. Again, we include nearest neighbours as one of the baselines. LeNet \cite{lecun1998gradient} architecture was used for all the baselines and \drocc for this experiment. \drocc consistently achieves the best performance on most classes, with gains as high as 20\% over DeepSVDD on some classes. An interesting observation is that for the classes Bird and Deer, Nearest Neighbour achieves competitive performance, beating all the other baselines. 

As discussed in Section~\ref{sec:rw}, \cite{golan2018, hendryks2019} use domain specific transformations like flip and rotations to perform the AD task. The performance of these approaches is heavily dependent on the interaction between transformations and the dataset.
They would suffer significantly in more realistic settings where the images of \textit{normal} class itself have been captured from multiple orientations. For example, even in CIFAR, for \textit{airplane} class, the accuracy is relatively low (DROCC is 7\% more accurate) as the images have airplanes in multiple angles. In fact, we try to mimic a more realistic scenario by augmenting the CIFAR-10 data with flips and small rotations of angle $\pm30^{\circ}$. Table \ref{tab:CIFAR-Aug} depicts the drop in performance of GEOM when augmentations are added in the CIFAR-10 dataset. For example, on the \textit{deer} class of CIFAR-10 dataset, GEOM has an AUC of 87.8\%, which falls to 65.8\% when augmented CIFAR-10 is used whereas \drocc’s performance remains the same ($\sim72\%$).

Next, we benchmark the performance of \drocc on high resolution images that require the use of large modern neural architectures. Table~\ref{tab:ImageNet} presents the results of our experiments on ImageNet. \drocc continues to achieve the best results amongst all the compared methods. Autoencoder fails drastically on this dataset, so we exclude comparisons. For DeepSVDD and \drocc, MobileNetv2 \cite{sandler2018mobilenetv2} architecture is used. We observe that for all classes, except golf ball, \drocc outperforms the baselines. For instance, on French-Horn vs. rest problem, \drocc is 23\% more accurate than DeepSVDD.

\begin{table}[tp!]
\centering
\caption{Comparing \drocc against GEOM \cite{golan2018} on CIFAR-10 data flipped and rotated by a small angle of $\pm30$ degree}
\resizebox{\columnwidth}{!}{
\scriptsize
\begin{tabular}{lllll}
\hline
\textbf{CIFAR-10 Class} & \textbf{\begin{tabular}[c]{@{}l@{}}GEOM \\ (No Aug)\end{tabular}} & \textbf{\begin{tabular}[c]{@{}l@{}}\ \drocc \\ (No Aug)\end{tabular}} & \textbf{\begin{tabular}[c]{@{}l@{}}GEOM \\ (with Aug)\end{tabular}} & \textbf{\begin{tabular}[c]{@{}l@{}}\drocc \\ (with Aug)\end{tabular}} \\ \hline
Airplane                & 74.7 $\pm$ 0.4                                                    & 81.6 $\pm$ 0.2                                                                     & 62.4 $\pm$ 1.7                                                      & 77.2 $\pm$ 1.2                                                                       \\
Automobile              & 95.7 $\pm$ 0.0                                                    & 76.7 $\pm$ 1.0                                                                     & 71.8 $\pm$ 1.2                                                      & 74.5 $\pm$ 1.8                                                                       \\
Bird                    & 78.1 $\pm$ 0.4                                                    & 66.7 $\pm$ 1.0                                                                     & 50.6 $\pm$ 0.5                                                      & 67.5 $\pm$ 1.0                                                                       \\
Cat                     & 72.4 $\pm$ 0.5                                                    & 67.1 $\pm$ 1.5                                                                     & 52.5 $\pm$ 0.7                                                      & 68.8 $\pm$ 2.3                                                                       \\
Deer                    & 87.8 $\pm$ 0.2                                                    & 73.6 $\pm$ 2.0                                                                     & 65.7 $\pm$ 1.7                                                      & 71.1 $\pm$ 2.9                                                                       \\
Dog                     & 87.8 $\pm$ 0.1                                                    & 74.4 $\pm$ 1.9                                                                     & 69.6 $\pm$ 1.3                                                      & 71.3 $\pm$ 0.4                                                                       \\
Frog                    & 83.4  $\pm$ 0.5                                                   & 74.4 $\pm$ 0.9                                                                     & 68.3 $\pm$ 1.1                                                      & 71.2 $\pm$ 1.6                                                                       \\
Horse                   & 95.5 $\pm$ 0.1                                                    & 71.4 $\pm$ 0.2                                                                     & 84.8 $\pm$ 0.8                                                      &  63.5 $\pm$ 3.5                                                                                    \\
Ship                    & 93.3 $\pm$ 0.0                                                    & 80.0 $\pm$ 1.7                                                                     & 79.6  $\pm$ 2.2                                                     & 76.4 $\pm$ 3.5                                                                       \\
Truck                   & 91.3 $\pm$ 0.1                                                    & 76.2 $\pm$ 0.7                                                                     & 85.7 $\pm$ 0.5                                                      & 74.0 $\pm$ 1.0                                                                       \\ \hline
\end{tabular}
}
\label{tab:CIFAR-Aug}
\end{table}

\begin{table}[t]
\centering
\caption{Average AUC (with standard deviation) for one-vs-all anomaly detection on ImageNet. \drocc consistently achieves the best performance for all but one class.}
\resizebox{\columnwidth}{!}{
\begin{tabular}{llll}
\hline
\textbf{ImageNet Class}        & \textbf{\begin{tabular}[c]{@{}l@{}}Nearest \\ Neighbor\end{tabular}} & \textbf{DeepSVDD}          & \textbf{\drocc (Ours)}     \\ \hline
Tench            & 65.57                      & 65.14 $\pm$ 1.03           & \textbf{70.19 $\pm$ 1.7}  \\
English Springer  & 56.37                      & 66.45 $\pm$ 3.16           & \textbf{70.45 $\pm$ 4.99} \\
Cassette Player & 47.7                       & 60.47 $\pm$ 5.35           & \textbf{71.17 $\pm$ 1}    \\
Chainsaw       & 45.22                      & 59.43 $\pm$ 4.13           & \textbf{68.63 $\pm$ 1.86} \\
Church           & 61.35                      & 56.31 $\pm$ 4.23           & \textbf{67.46 $\pm$ 4.17} \\
French Horn     & 50.52                      & 53.06 $\pm$ 6.52           & \textbf{76.97 $\pm$ 1.67} \\
Garbage Truck   & 54.2                       & 62.15 $\pm$ 4.39           & \textbf{69.06 $\pm$ 2.34} \\
Gas Pump        & 47.43                      & 56.66 $\pm$ 1.49           & \textbf{69.94 $\pm$ 0.57} \\
Golf Ball       & 70.36                      & \textbf{72.23  $\pm$ 3.43} & 70.72 $\pm$ 3.83          \\
Parachute        & 75.87                      & 81.35 $\pm$ 3.73           & \textbf{93.5 $\pm$ 1.41}  \\ \hline
\end{tabular}}
\label{tab:ImageNet}
\end{table}

\subsection{One-class Classification with Limited Negatives (\lfoc)}\label{sec:exp_fpr}

Recall that the goal in \lfoc is to learn a classifier that is accurate for both, the in-sample positive (or normal) class points and for the arbitrary out-of-distribution (OOD) negatives. Naturally, the key metric for this problem is False Positive Rate (FPR). In our experiments, we bound any method to have FPR to be smaller than a threshold, and under that constraint, we measure it's recall value, i.e., the fraction of true positives that are correctly predicted.

We compare \drocclf against the following baselines: a) Standard binary classifier: that is, we ignore the challenge of OOD negatives and treat the problem as a standard classification task, b) DeepSAD \cite{ruff2020deep}: a semi-supervised anomaly detection method but it is not explicitly designed to handle negatives that are very close to positives (OOD negatives) and c) \droccoe : Outlier exposure type extension where \drocc's anomaly detection loss (based on using Euclidean distance as a local distance measure over the manifold) is combined with standard cross-entropy loss over the given limited negative data.
Similar to the anomaly detection experiments, we use the same underlying network architecture across all the baselines. \par 

\subsubsection{Results}
\label{subsec:lfoc_res}
\textbf{Synthetic Data}: We sample $1024$ points in $\R^{10}$, where the first two coordinates are sampled from the 2D-sine wave, as in the previous section. Coordinates $3$ to $10$ are sampled from the spherical Gaussian distribution. Note that due to the 8 noisy dimensions, \drocc would be forced to set $r=\sqrt{d}$ where $d=10$, while the true low-dimensional manifold is restricted to only two dimensions. Consequently, it learns an inaccurate boundary as shown in Figure~\ref{fig:synthetic_c} and is similar to the boundary learned by OC-SVM and DeepSVDD; points that are predicted to be positive by \drocc are colored blue.
In contrast, \drocclf is able to learn that only the first two coordinates are useful for the distinction between positives and negatives, and hence is able to learn a skewed distance function, leading to an accurate decision boundary (see Figure~\ref{fig:synthetic_d}). 

\textbf{MNIST 0 vs. 1 Classification}: We consider an experimental setup on MNIST dataset, where the training data consists of Digit $0$, the {\em normal} class, and the Digit $1$ as the anomaly. During evaluation, in addition to samples from training distribution, we also have \textit{half zeros}, which act as challenging OOD points (close negatives). These \textit{half zeros} are generated by randomly masking 50\% of the pixels (Figure \ref{fig:mnist_cn}). BCE performs poorly, with a recall of $54\%$ only at a fixed FPR of $3\%$. \droccoe gives a recall value of $98.16\%$ outperforming DeepSAD by a margin of 7\%, which gives a recall value of $90.91\%$. \drocclf provides further improvement with a recall of $99.4\%$ at $3\%$ FPR.

\begin{figure}[t]%
    \centering
    \subfloat{{\includegraphics[width=6cm]{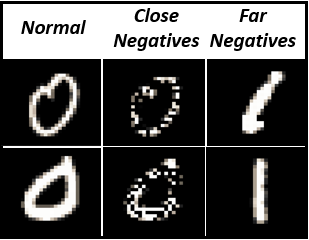}}}
    \caption{Sample postives, negatives and close negatives for MNIST digit 0 vs 1 experiment (OCLN).}
\label{fig:mnist_cn}
\end{figure}

\begin{figure}[t]%
    \centering
    \subfloat{{\includegraphics[width=8cm]{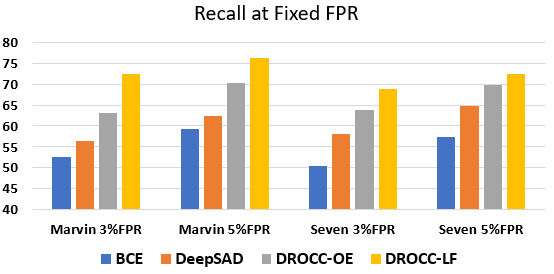}}}
    \caption{OCLN on Audio Commands: Comparison of Recall for key words — “Marvin” and “Seven” when the False Positive Rate(FPR) is fixed to be 3\% and 5\%. \drocclf is consistently about 10\% more accurate than all the baseline}
\label{fig:lfoc_result}
\vspace{5pt}
\end{figure}

\textbf{Wakeword Detection}: Finally, we evaluate \drocclf on the practical problem of wakeword detection with low FPR against arbitrary OOD negatives. To this end, we identify a keyword, say ``Marvin" from the audio commands dataset \cite{warden2018speech} as the {\em positive} class, and the remaining 34 keywords are labeled as the negative class. For training, we sample points uniformly at random from the above mentioned dataset. However, for evaluation, we sample positives from the train distribution, but negatives contain a few challenging OOD points as well. Sampling challenging negatives itself is a hard task and is the key motivating reason for studying the problem. So, we manually list close-by keywords to {\em Marvin} such as: {\em Mar}, {\em Vin}, {\em Marvelous} etc. We then generate audio snippets for these keywords via a speech synthesis tool \footnote{\url{https://azure.microsoft.com/en-in/services/cognitive-services/text-to-speech/}} with a variety of accents. 

Figure ~\ref{fig:lfoc_result} shows that for 3\% and 5\% FPR settings, \drocclf is significantly more accurate than the baselines. For example, with FPR=$3\%$, \drocclf is 10\% more accurate than the baselines. We repeated the same experiment with the keyword: {\em Seven}, and observed a similar trend. See Table~\ref{tab:synth_keywords} in Appendix for the list of the close negatives which were synthesized for each of the keywords. In summary, \drocclf is able to generalize well against negatives that are ``close" to the true positives even when such negatives were not supplied with the training data. 
\section{Conclusions}\vspace*{-3pt}
\label{sec:conc}
We introduced \drocc method for deep anomaly detection. It models normal data points using a low-dimensional manifold, and hence can compare close point via Euclidean distance. Based on this intuition, \drocc's optimization is formulated as a saddle point problem which is solved via standard gradient descent-ascent algorithm. We then extended \drocc to \lfoc problem where the goal is to generalize well against {\em arbitrary} negatives, assuming positive class is well sampled and a small number of negative points are also available. Both the methods perform significantly better than strong baselines, in their respective problem settings. For computational efficiency, we simplified the projection set for both the methods which can perhaps slow down the convergence of the two methods. Designing optimization algorithms that can work with the stricter set is an exciting research direction. Further, we would also like to rigorously analyse \drocc, assuming enough samples from a low-curvature manifold.  Finally,  as \lfoc is an exciting problem that routinely comes up in a variety of real-world applications, we would like to apply \drocclf to a few high impact scenarios.

\subsection*{Acknowledgments}
\label{sec:ack}
We are grateful to Aditya Kusupati, Nagarajan Natarajan, Sahil Bhatia and Oindrila Saha for helpful discussions and feedback. AR was funded by an Open Philanthropy AI Fellowship and Google PhD Fellowship in Machine Learning.  

\bibliographystyle{icml2020}
\bibliography{main}
\clearpage
\appendix
\addtolength{\dbltextfloatsep}{0.1in}
\addtolength{\textfloatsep}{0.1in}
\section{\lfoc}
\label{app:lfoc}
\subsection{\drocclf Proof}
\begin{proof}[Proof of Proposition~\ref{prop:lfoc}]
    Recall the problem: 
    \begin{equation*}\min_{\tilde{x}} \|\tilde{x}-z\|^2,\ s.t.,\  r^2\leq \|\tilde{x}-x\|_\Sigma^2 \leq \gamma^2 r^2.\end{equation*}
    Note that both the constraints cannot be active at the same time, so we can consider either $r^2\leq \|\tilde{x}-x\|_\Sigma^2$ constraint or $\|\tilde{x}-x\|_\Sigma^2 \leq \gamma^2 r^2$. Below, we give calculation when the former constraint is active, later's proof follows along same lines. 
    
    Let $\tau\leq 0$ be the Lagrangian multiplier, then the Lagrangian function of the above problem is given by: 
    $$L(\tilde{x},\tau)=\|\tilde{x}-z\|^2+\tau (\|\tilde{x}-x\|_\Sigma^2 -r^2).$$
    Using KKT first-order necessary condition \cite{boyd}, the following should hold for any optimal solution $\tilde{x},\tau$:
    $$\nabla_{\tilde{x}} L(\tilde{x},\tau)=0.$$
    That is, 
    $$\tilde{x}=(I+\tau \Sigma)^{-1}(z+\tau \cdot\Sigma x)=x+(I+\tau \cdot\Sigma)^{-1}\delta,$$
    where $\delta=z-x$. This proves the first part of the lemma. 
    
    Now, by using primal and dual feasibility required by the KKT conditions, we have: 
    \begin{equation*}\min_{\tau\leq 0} \|\tilde{x}-z\|^2,\ s.t.,\  \|\tilde{x}-x\|_\Sigma^2 \geq r^2,\end{equation*}
    where $\tilde{x}=(I+\tau \Sigma)^{-1}(z+\tau \cdot\Sigma x)=x+(I+\tau \cdot\Sigma)^{-1}\delta$. The lemma now follows by substituting $\tilde{x}$ above and by using the fact that $\Sigma$ is a diagonal matrix with $\Sigma(i,i)=\sigma_i$. 
\end{proof}

\subsection{\drocclf Algorithm}
See Algorithm Box ~\ref{alg:drocclf-nn}.

\begin{figure}
    \begin{minipage}{.49\textwidth}
        \begin{algorithm}[H]
            \algrenewcommand\algorithmicprocedure{}
            \caption{Training neural networks via \drocclf} \label{alg:drocclf-nn} \textbf{Input:} Training data $D=[(x_1, y_1),(x_2, y_2), \dots, (x_n,y_n)]$.  
            
            \textbf{Parameters:} Radius $r$, $\lambda\geq 0$, $\mu\geq 0$, step-size $\eta$, number of gradient steps $\ngrad$, number of initial training steps $\ninit$.
            
            \textbf{Initial steps:} For $B = 1, \hdots \ninit$
            
            \hspace*{\algorithmicindent} $X_B$: Batch of training inputs
            
            \hspace*{\algorithmicindent} $\theta = \theta - \text{Gradient-Step}\Big(\sum \limits_{\substack{(x,y) \\ \in X_B}} \ell(f_\theta(x), y) \Big)$
            
            \textbf{DROCC steps:} For $B = \ninit, \hdots \ninit + N$ 
            
            \hspace*{\algorithmicindent} $X_B$: Batch of \textit{normal} training inputs ($y=1$)
            
            
            \hspace*{\algorithmicindent} $\forall x \in X_B: h \sim \sN(0, I_{d})$

            \hspace*{\algorithmicindent} \textbf{Adversarial search:} For $i = 1, \hdots \ngrad$
            
            \hspace*{\algorithmicindent} \hspace*{\algorithmicindent} 1. $\ell(h) = \ell(f_\theta(x + h), -1)$
            
            \hspace*{\algorithmicindent} \hspace*{\algorithmicindent} 2. $h = h + \eta \frac{\nabla_h \ell(h)}{\| \nabla_h \ell(h) \|}$
            
            \hspace*{\algorithmicindent} \hspace*{\algorithmicindent} 3. $h = \text{Projection given by Proposition~\ref{prop:lfoc}}(\delta=h)$
            
            \hspace*{\algorithmicindent} $\ell^{itr} =  \lambda\| \theta \|^2 + \sum \limits_{\substack{(x,y) \\ \in X_B}} \ell(f_\theta(x), y) +\mu  \ell(f_\theta(x + h), -1)  $

            \hspace*{\algorithmicindent} $\theta = \theta - \text{Gradient-Step}(\ell^{itr})$  
        \end{algorithm}
    \end{minipage}\hspace*{5pt}
    \begin{minipage}{.49\textwidth}

    \end{minipage}
\end{figure}

\begin{figure}[t]
    \label{fig:sphere_expt_fig}
    \centering
    \subfloat[\label{fig:sphere_manifold}]{{\includegraphics[width=3cm]{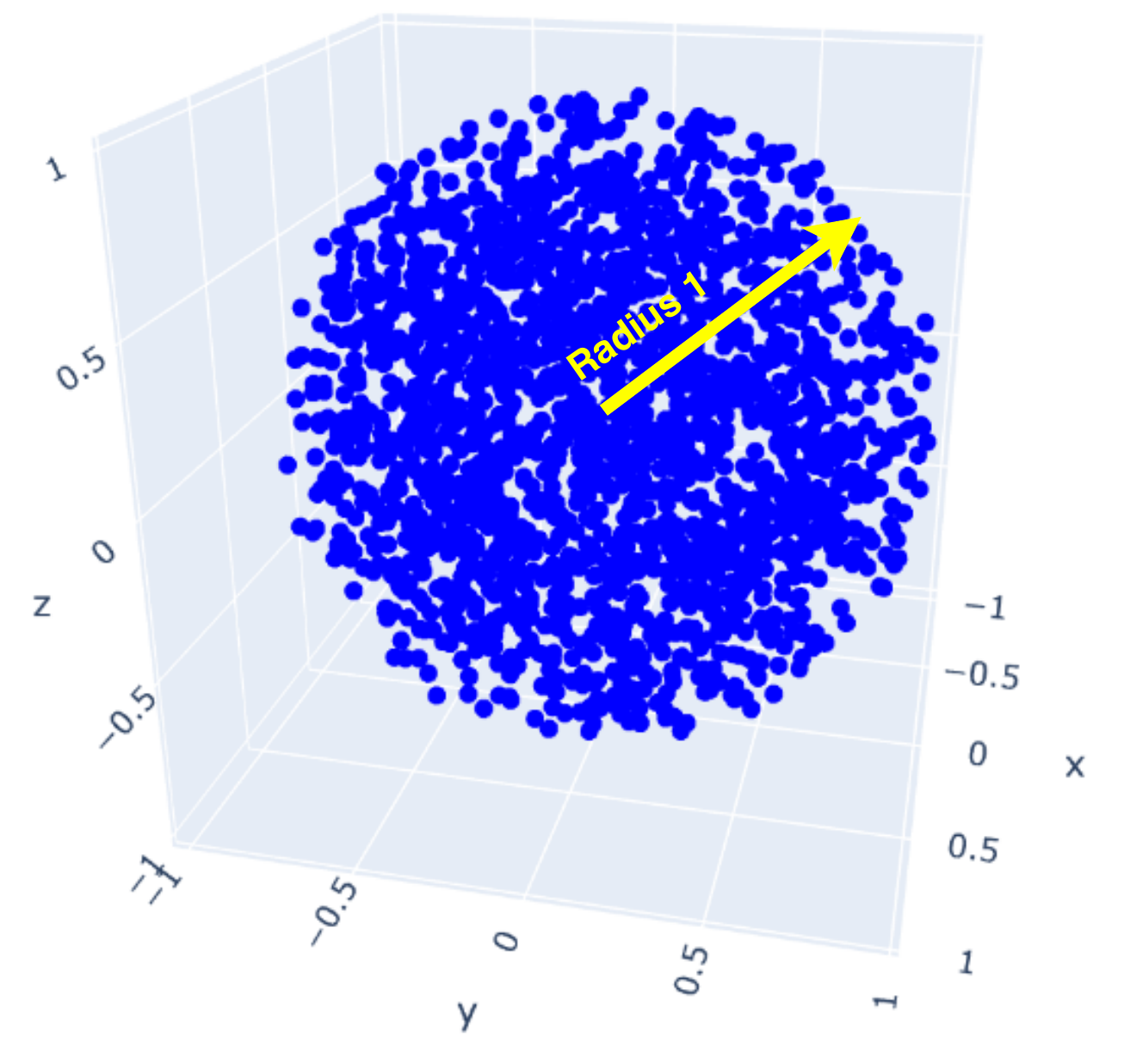}}}\qquad
    \subfloat[\label{fig:sphere_with_neg}]{{\includegraphics[width=3cm]{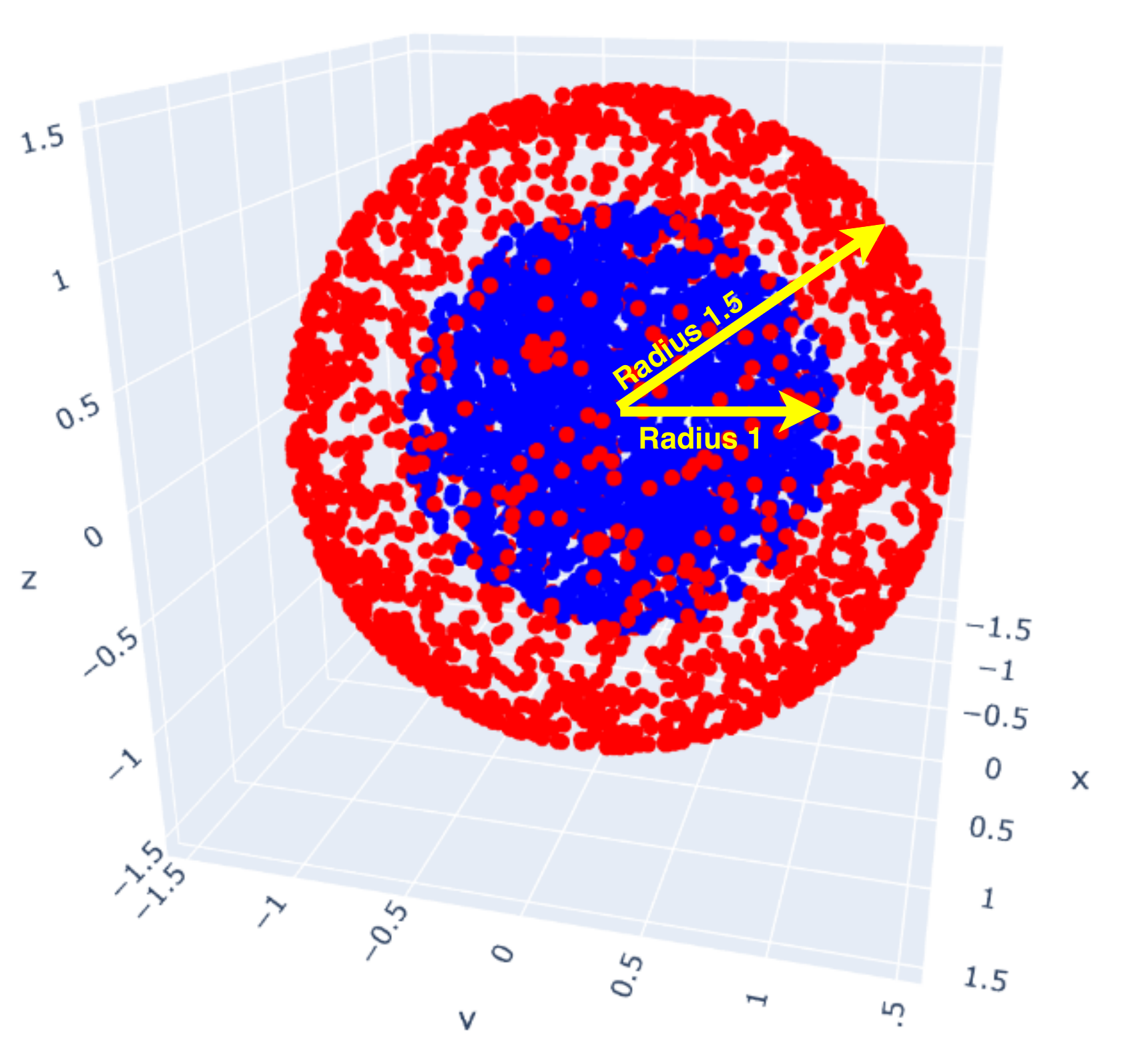}}}\\
    \caption{(a) Spherical manifold (a unit sphere) that captures the normal data distribution. Points are uniformly sampled from the volume of the unit sphere. (b) OOD points (red) are sampled on the \textit{surface} of a sphere of varying radius. Table~\ref{tab:synthetic_sphere} shows AUC values with varying radius. }
\end{figure}

\begin{table*}[t]
    \centering
    \caption{Average AUC for Spherical manifold experiment (Section~\ref{sec:sphere_exp}). Normal points are sampled uniformly from the volume of a unit sphere and OOD points are sampled from the \textit{surface} of a unit sphere of varying radius (See Figure~\ref{fig:sphere_with_neg}). Again \drocc outperforms all the baselines when the OOD points are quite close to the normal distribution.}
    {\scriptsize
        \begin{tabular}{llllll}
            \hline
            \textbf{Radius}        & \textbf{\begin{tabular}[c]{@{}l@{}}Nearest \\ Neighbor\end{tabular}} & \textbf{OC-SVM} & \textbf{AutoEncoder} & \textbf{DeepSVDD}          & \textbf{\drocc (Ours)}     \\ \hline
            1.2            & 100 $\pm$ 0.00  & 92.00 $\pm$ 0.00 & 91.81 $\pm$ 2.12 & 93.26 $\pm$ 0.91 & 99.44 $\pm$ 0.10  \\
            1.4            & 100 $\pm$ 0.00  & 92.97 $\pm$ 0.00 & 97.85 $\pm$ 1.41 & 98.81 $\pm$ 0.34 & 99.99 $\pm$ 0.00  \\
            1.6            & 100 $\pm$ 0.00  & 92.97 $\pm$ 0.00 & 99.92 $\pm$ 0.11 & 99.99 $\pm$ 0.00 & 100.00 $\pm$ 0.00  \\
            1.8            & 100 $\pm$ 0.00  & 91.87 $\pm$ 0.00 & 99.98 $\pm$ 0.04 & 100.00 $\pm$ 0.00 & 100.00 $\pm$ 0.00  \\
            2.0            & 100 $\pm$ 0.00  & 91.83 $\pm$ 0.00 & 100 $\pm$ 0.00 & 100.00 $\pm$ 0.00 & 100.00 $\pm$ 0.00  \\
    \end{tabular}}
    \label{tab:synthetic_sphere}
\end{table*}

\section{Synthetic Experiments}
\label{app:model-vis}
\subsection{1-D Sine Manifold}
\label{sec:sine_exp}
In Section \ref{sec:exp_sine} we presented results on a synthetic dataset of 1024 points sampled from a 1-D sine wave (See Figure \ref{fig:synthetic_a}). We compare \drocc to other anomaly detection methods by plotting the decision boundaries on this same dataset. Figure~\ref{fig:model_vis} shows the decision boundary for a) \drocc b) OC-SVM with RBF kernel c) OC-SVM with 20-degree polynomial kernel  d) DeepSVDD. 
All methods are trained only on positive points from the 1-D manifold. 
\begin{figure}[t!]
    \centering
    \subfloat[\label{fig:synthetic_b1}]{{\includegraphics[width=3.3cm]{figures/learnt_sine.png}}}\qquad
    \subfloat[\label{fig:sine_rbf}]{{\includegraphics[width=3.3cm]{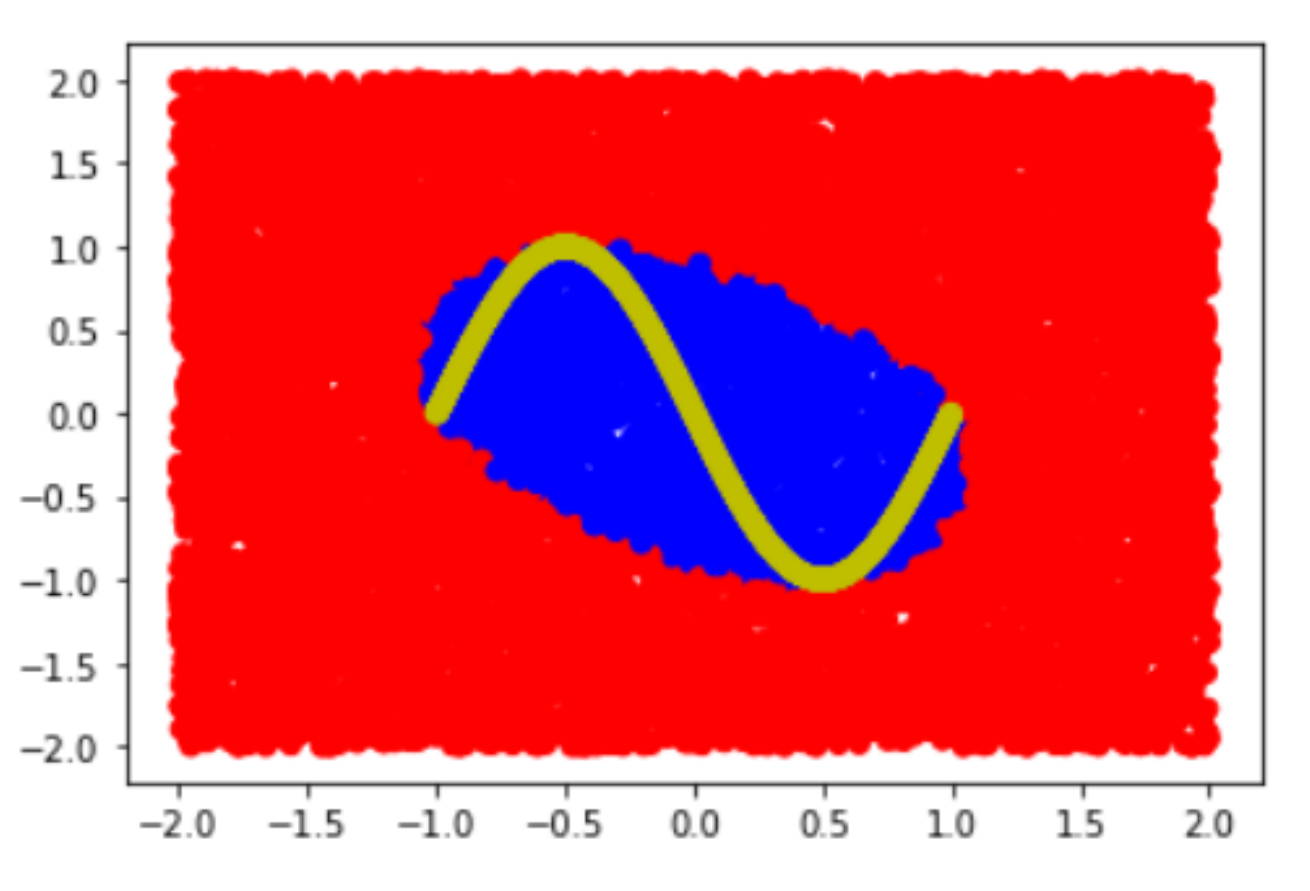}}}\\
    \subfloat[\label{fig:sine_poly}]{{\includegraphics[width=3.3cm]{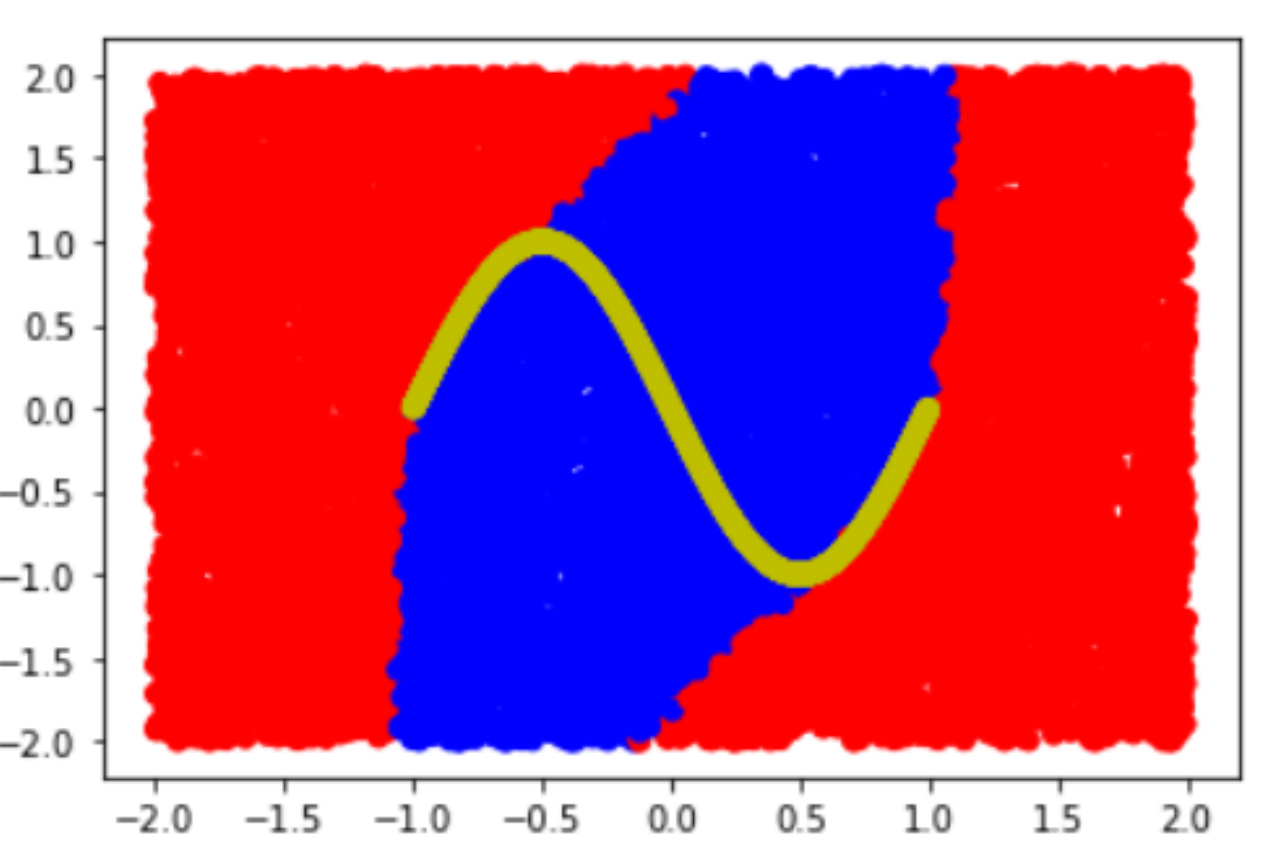}}}\qquad
    \subfloat[\label{fig:sine_deepsvdd}]{{\includegraphics[width=3.1cm]{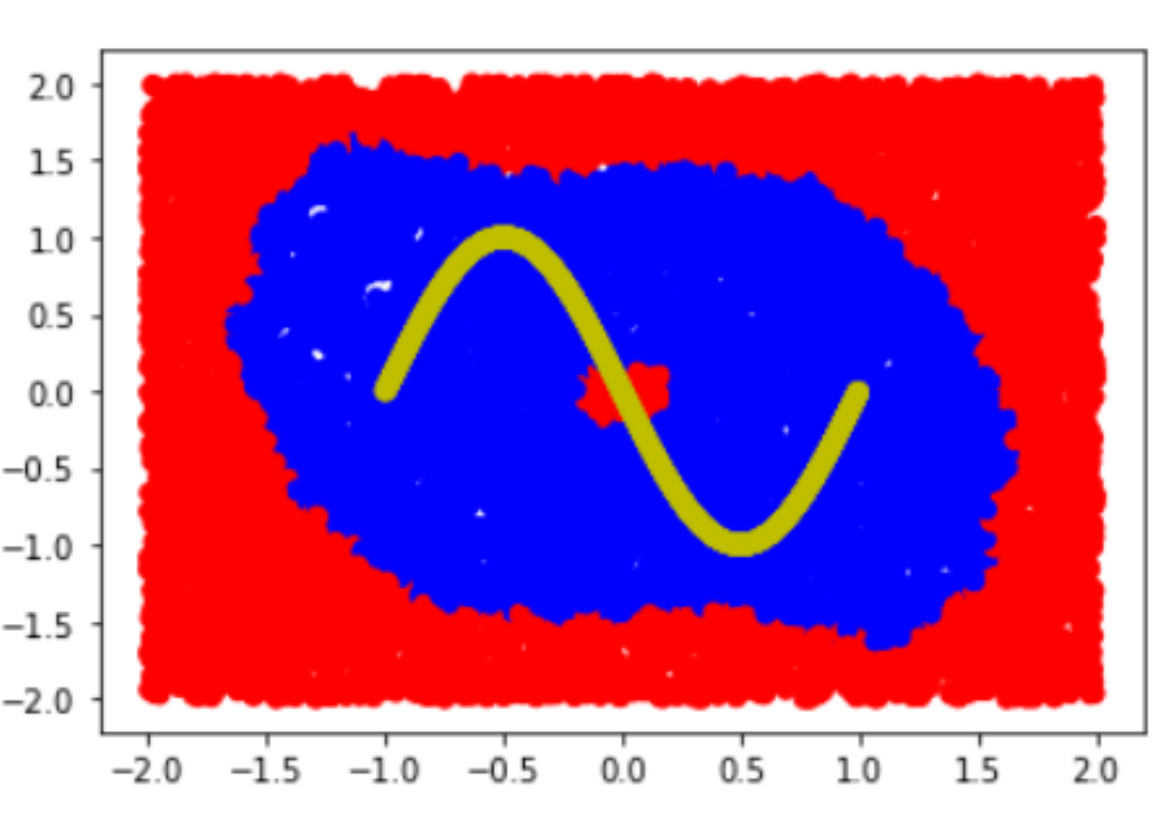}}}
    \caption{(a) Decision boundary of \drocc trained only on the positive points lying on the 1-D sine manifold in Figure \ref{fig:synthetic_a}. Blue represents points classified as normal and red classified as abnormal. (b) Decision boundary of classical OC-SVM using RBF kernel and same experiment settings as in (a). Yellow sine wave just shows the underlying train data. (c) Decision boundary of classical OC-SVM using a 20-degree polynomial kernel. (d) Decision boundary of DeepSVDD.}
\label{fig:model_vis}
\end{figure}

We further evaluate these methods for varied sampling of negative points near the positive manifold. Negative points are sampled from a 1-D sine manifold vertically displaced in both directions (See Figure \ref{fig:displaced_manifold}). Table~\ref{tab:synthetic} compares \drocc against various baselines on this dataset.

\begin{figure}[t!]
    \centering
    \subfloat{{\includegraphics[width=4cm]{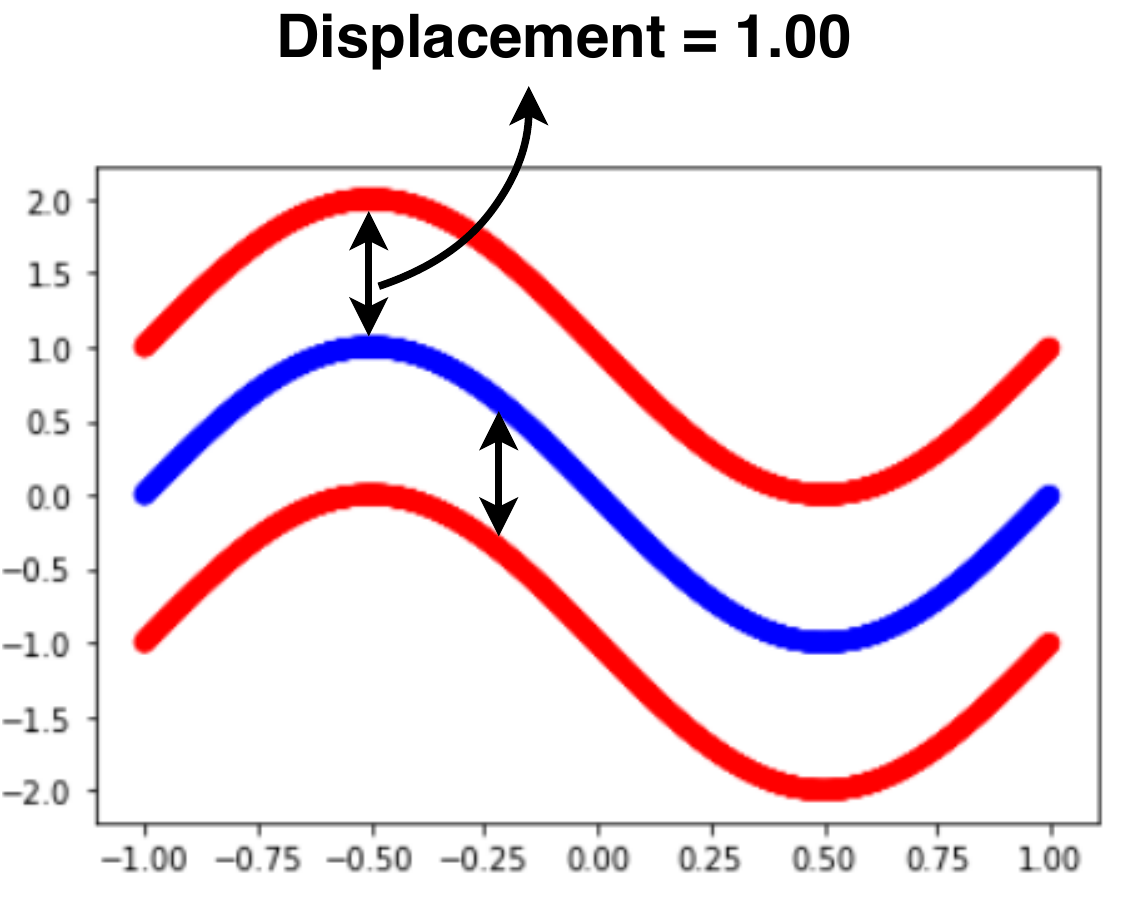}}}\qquad
    \caption{Illustration of the negative points sampled at various displacements of the sine wave; used for reporting the AUC values in the Table \ref{tab:synthetic}. In this figure, vertical displacement is 1.0. Blue represents the positive points (also the training data) and red represents the negative/OOD points}
\label{fig:displaced_manifold}
\end{figure}

\begin{table*}[t]
\centering
\caption{Average AUC for the synthetic 1-D Sine Wave manifold experiment (Section~\ref{sec:sine_exp}). Normal points are sampled from a sine wave and OOD points from a vertically displaced manifold (See Figure~\ref{fig:displaced_manifold}). The results demonstrate that only \drocc is able to capture the manifold tightly}

{\scriptsize
\begin{tabular}{llllll}
\hline
\textbf{\begin{tabular}[c]{@{}l@{}}Vertical \\ Displacement\end{tabular}}        & \textbf{\begin{tabular}[c]{@{}l@{}}Nearest \\ Neighbor\end{tabular}} & \textbf{OC-SVM} & \textbf{AutoEncoder} & \textbf{DeepSVDD}          & \textbf{\drocc (Ours)}     \\ \hline
0.2            & 100 $\pm$ 0.00  & 56.99 $\pm$ 0.00 & 52.48 $\pm$ 1.15 & 65.91 $\pm$ 0.64 & 96.80 $\pm$ 0.65  \\
0.4            & 100 $\pm$ 0.00  & 68.84 $\pm$ 0.00 & 58.59 $\pm$ 0.61 & 78.18 $\pm$ 1.67 & 99.31 $\pm$ 0.80  \\
0.6            & 100 $\pm$ 0.00  & 76.95 $\pm$ 0.00 & 66.59 $\pm$ 1.21 & 82.85 $\pm$ 1.96 & 99.92 $\pm$ 0.11  \\
0.8            & 100 $\pm$ 0.00  & 81.73 $\pm$ 0.00 & 77.42 $\pm$ 3.62 & 86.26 $\pm$ 1.69 & 99.98 $\pm$ 0.01  \\
1.0            & 100 $\pm$ 0.00  & 88.18 $\pm$ 0.00 & 86.14 $\pm$ 2.52 & 90.51 $\pm$ 2.62 & 100 $\pm$ 0.00  \\
2.0            & 100 $\pm$ 0.00  & 98.56 $\pm$ 0.00 & 100 $\pm$ 0.00 & 100 $\pm$ 0.00 & 100 $\pm$ 0.00  \\
\end{tabular}}
\label{tab:synthetic}
\end{table*}

\subsection{Spherical Manifold}
\label{sec:sphere_exp}
OC-SVM and DeepSVDD try to find a minimum enclosing ball for the whole set of positive points, while DROCC assumes that the true points low on a low dimensional manifold.
We now test these methods on a different synthetic dataset: spherical manifold where the positive points are within a sphere, as shown in Figure~\ref{fig:sphere_manifold}. Normal/Positive points are sampled uniformly from the volume of the unit sphere. Table~\ref{tab:synthetic_sphere} compares \drocc against various baselines when the OOD points are sampled on the \textit{surface} of a sphere of varying radius (See Figure~\ref{fig:sphere_with_neg}). \drocc again outperforms all the baselines even in the case when minimum enclosing ball would suit the best. Suppose instead of neural networks, we were operating with purely linear models, then \drocc also essentially finds the minimum enclosing ball (for a suitable radius $r$). If $r$ is too small, the training doesn't converge since there is no separating boundary). Assuming neural networks are implicitly regularized to find the simplest boundary, \drocc with neural networks also learns essentially a minimum enclosing ball in this case, however, at a slightly larger radius. Therefore, we get $100\%$ AUC only at radius $1.6$ rather than $1 + \epsilon$ for some very small $\epsilon$.

\begin{figure*}[t]
    \hspace*{\fill}%
    \subfloat[\label{fig:abl_r_airplane}]{{
            \resizebox{0.31\linewidth}{!}{
                \begin{tikzpicture}
                \begin{axis}[
                title={CIFAR-10 Airplane: AUC vs Radius},
                xlabel={Radius},
                ymax=90,
                ymin=50,
                ylabel={AUC}
                ]
                \addplot [color=blue,mark=o,]
                     plot [error bars/.cd, y dir = both, y explicit]
                     table[y error index=2]{figures/ablation_cifar_0_rad.dat};
                \end{axis}
                \end{tikzpicture}
    }}}
    \hfill
        \subfloat[\label{fig:abl_r_deer}]{{
            \resizebox{0.31\linewidth}{!}{
                \begin{tikzpicture}
                \begin{axis}[
                title={CIFAR-10 Deer: AUC vs Radius},
                xlabel={Radius},
                ymax=90,
                ymin=45,
                ylabel={AUC}
                ]
                \addplot [color=blue,mark=o,]
                     plot [error bars/.cd, y dir = both, y explicit]
                     table[y error index=2]{figures/ablation_cifar_4_rad.dat};
                \end{axis}
                \end{tikzpicture}
    }}}
    \hfill
    \subfloat[\label{fig:abl_r_truck}]{{
            \resizebox{0.30\linewidth}{!}{
                \begin{tikzpicture}
                \begin{axis}[
                title={CIFAR-10 Truck: AUC vs Radius},
                xlabel={Radius},
                ymax=90,
                ymin=50,
                ylabel={AUC}
                ]
                \addplot [color=blue,mark=o,]
                     plot [error bars/.cd, y dir = both, y explicit]
                     table[y error index=2]{figures/ablation_cifar_9_rad.dat};
                \end{axis}
                \end{tikzpicture}
    }}}
    \hspace*{\fill}%
    \caption{Ablation Study : Variation in the performance \drocc when $ r$ (with $\gamma=1$) is changed from the optimal value. }
    \label{fig:ablation_cifar_radius}
\end{figure*}
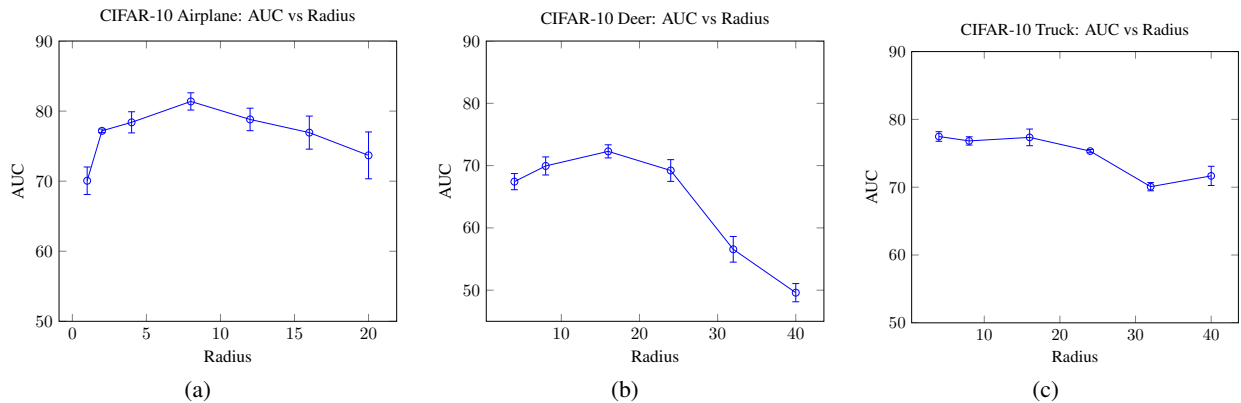

\begin{figure*}[t]
    \hspace*{\fill}%
    \subfloat[\label{fig:abl_lamda_airplane}]{{
            \resizebox{0.30\linewidth}{!}{
                \begin{tikzpicture}
                \begin{axis}[
                title={CIFAR-10 Airplane: AUC vs $\mu$},
                xlabel={ $\mu$},
                ymax=90,
                ymin=50,
                ylabel={AUC}
                ]
                \addplot [color=blue,mark=o,]
                     plot [error bars/.cd, y dir = both, y explicit]
                     table[y error index=2]{figures/ablation_cifar_0_lambda.dat};
                \end{axis}
                \end{tikzpicture}
    }}}
    \hfill
        \subfloat[\label{fig:abl_lamda_deer}]{{
            \resizebox{0.30\linewidth}{!}{
                \begin{tikzpicture}
                \begin{axis}[
                title={CIFAR-10 Deer: AUC vs  $\mu$},
                xlabel={ $\mu$},
                ymax=90,
                ymin=50,
                ylabel={AUC}
                ]
                \addplot [color=blue,mark=o,]
                     plot [error bars/.cd, y dir = both, y explicit]
                     table[y error index=2]{figures/ablation_cifar_4_lambda.dat};
                \end{axis}
                \end{tikzpicture}
    }}}
    \hfill
    \subfloat[\label{fig:abl_lamda_truck}]{{
            \resizebox{0.30\linewidth}{!}{
                \begin{tikzpicture}
                \begin{axis}[
                title={CIFAR-10 Truck: AUC vs  $\mu$},
                xlabel={ $\mu$},
                ymax=90,
                ymin=50,
                ylabel={AUC}
                ]
                \addplot [color=blue,mark=o,]
                     plot [error bars/.cd, y dir = both, y explicit]
                     table[y error index=2]{figures/ablation_cifar_9_lambda.dat};
                \end{axis}
                \end{tikzpicture}
    }}}
    \hspace*{\fill}%
    \caption{Ablation Study : Variation in the performance of \drocc with $\mu$ \eqref{eq:drocc-loss} which is the weightage given to the loss from adversarially sampled negative points }
    \label{fig:ablation_cifar_lambda}
\end{figure*}
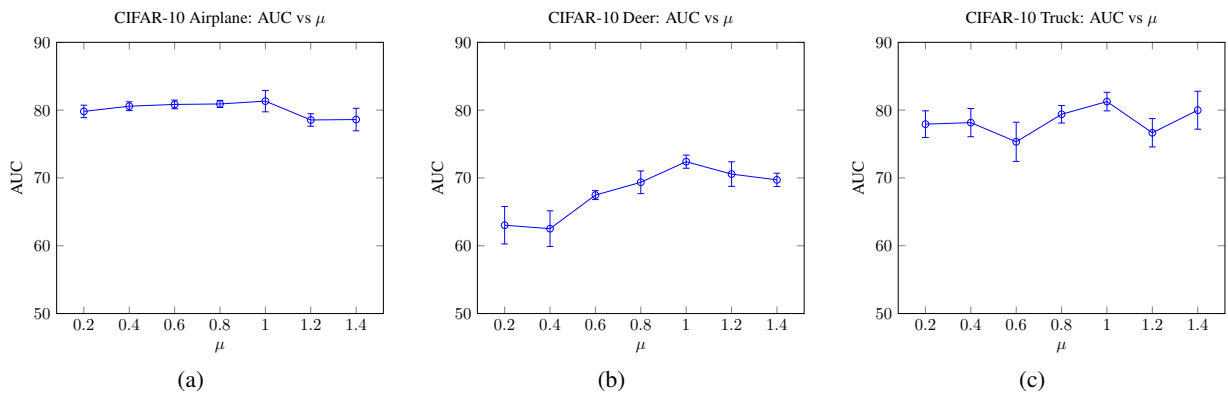

\begin{figure}[t]%
    \centering
    \subfloat{{\includegraphics[width=8cm]{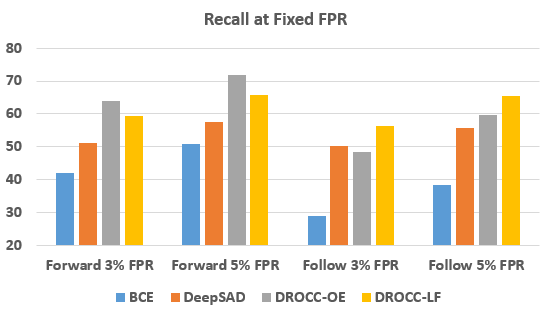}}}
    \caption{OCLN on Audio Commands: Comparison of Recall for key words — “Forward” and “Follow” when the False Positive Rate(FPR) is fixed to be 3\% and 5\%.}
\label{fig:lfoc_sup}
\vspace{5pt}
\end{figure}

\begin{table}[t!]
    \centering
    \caption{Ablation Study on CIFAR-10: Sampling negative points randomly in the set $\negset{i}(r)$ (\droccrand) instead of gradient ascent (\drocc).}
    {\scriptsize
        \begin{tabular}{llll}
            \hline
            \textbf{CIFAR Class} &  \textbf{\begin{tabular}[c]{@{}l@{}}One-Class \\ Deep SVDD\end{tabular}} & \textbf{\drocc}  & \textbf{\droccrand}  \\ \hline
            Airplane  & 61.7$\pm$4.1 & 81.66 $\pm$ 0.22 &  79.67 $\pm$ 2.09  \\
            Automobile  & 65.9$\pm$2.1 & 76.74 $\pm$ 0.99 &  73.48 $\pm$ 1.44  \\
            Bird  & 50.8$\pm$0.8 & 66.66 $\pm$ 0.96 &  62.76 $\pm$ 1.59  \\
            Cat  & 59.1$\pm$1.4 & 67.13 $\pm$ 1.51 &  67.33 $\pm$ 0.72  \\
            Deer  & 60.9$\pm$1.1 & 73.62 $\pm$ 2.00 &  56.09 $\pm$ 1.19  \\
            Dog  & 65.7$\pm$2.5 & 74.43 $\pm$ 1.95 &  65.88 $\pm$ 0.64  \\
            Frog  & 67.7$\pm$2.6 & 74.43 $\pm$ 0.92 &  74.82 $\pm$ 1.77  \\
            Horse  & 67.3$\pm$0.9 & 71.39 $\pm$ 0.22 &  62.08 $\pm$ 2.03  \\
            Ship  & 75.9$\pm$1.2 & 80.01 $\pm$ 1.69 &  80.04 $\pm$ 1.71  \\
            Truck  & 73.1$\pm$1.2 & 76.21 $\pm$ 0.67 &  70.80 $\pm$ 2.73  \\
            \hline
    \end{tabular}}
    \label{tab:cifar_random}
\end{table}

\begin{table}[t]
    \centering
    \caption{Synthesized near-negatives for keywords in Audio Commands}
    {\scriptsize
        \begin{tabular}{lllll}
            \hline
            \textbf{Marvin} & \textbf{Forward} & \textbf{Seven} & \textbf{Follow} \\ \hline
            mar             & for              & one            & fall  \\
            marlin          & fervor           & eleven         & fellow \\
            arvin           & ward             & heaven         & low \\
            marvik          & reward           & when           & hollow \\
            arvi            & onward           & devon          & wallow \\ \hline
        \end{tabular}
    }
    \label{tab:synth_keywords}
\end{table}

\begin{table}[t]
    \centering
    \caption{Hyperparameters: Tabular Experiments}
    {\scriptsize
        \begin{tabular}{llllll}
            \hline
            \textbf{Dataset} & \textbf{Radius} & \textbf{$\mu$} & \textbf{Optimizer} & \textbf{\begin{tabular}[c]{@{}l@{}}Learning\\ Rate\end{tabular}} & \textbf{\begin{tabular}[c]{@{}l@{}}Adversarial\\ Ascent \\ Step Size\end{tabular}} \\ \hline
            
            Abalone                                                                    & 3              & 1.0               & Adam               & $10^{-3}$                                                              & 0.01                                                                                \\
            Arrhythmia                                                                    & 16              & 1.0               & Adam               & $10^{-4}$                                                              & 0.01                                                                                \\
            Thyroid                                                                   & 2.5              & 1.0             & Adam               & $10^{-3}$                                                              & 0.01                                                                                \\ \hline
        \end{tabular}
    }
    \label{tab:params_tabular}
\end{table}

\begin{table}[t]
    \centering
    \caption{Hyperparameters: CIFAR-10}
    {\scriptsize
        \begin{tabular}{llllll}
            \hline
            \textbf{Class} & \textbf{Radius} & \textbf{$\mu$} & \textbf{Optimizer} & \textbf{\begin{tabular}[c]{@{}l@{}}Learning \\ Rate\end{tabular}} & \textbf{\begin{tabular}[c]{@{}l@{}}Adversarial \\ Ascent \\ Step Size\end{tabular}} \\ \hline
            Airplane       & 8               & 1             & Adam               & 0.001                                                             & 0.001                                                                               \\
            Automobile     & 8              & 0.5             & SGD            & 0.001                                                              & 0.001                                                                               \\
            Bird           & 40              & 0.5             & Adam               & 0.001                                                             & 0.001                                                                               \\
            Cat            & 28              & 1             & SGD            & 0.001                                                             & 0.001                                                                               \\
            Deer           & 32              & 1               & SGD            & 0.001                                                             & 0.001                                                                               \\
            Dog            & 24              & 0.5             & SGD            & 0.01                                                              & 0.001                                                                               \\
            Frog           & 36              & 1             & SGD               & 0.001                                                              & 0.01                                                                               \\
            Horse          & 32              & 0.5             & SGD            & 0.001                                                              & 0.001                                                                               \\
            Ship           & 28              & 0.5               & SGD               & 0.001                                                              & 0.001                                                                               \\
            Truck          & 16              & 0.5             & SGD            & 0.001                                                             & 0.001                                                                               \\ \hline
        \end{tabular}
    }
    \label{tab:params_cifar}
\end{table}

\begin{table}[t]
    \centering
    \caption{Hyperparameters: ImageNet}
    {\scriptsize
        \begin{tabular}{llllll}
            \hline
            \textbf{Class}    & \textbf{Radius} & \textbf{$\mu$} & \textbf{Optimizer} & \textbf{\begin{tabular}[c]{@{}l@{}}Learning\\ Rate\end{tabular}} & \textbf{\begin{tabular}[c]{@{}l@{}}Adversarial \\ Ascent \\ Step Size\end{tabular}} \\ \hline
            Tench             & 30              & 1               & SGD            & 0.01                                                             & 0.001                                                                               \\
            English\_springer & 16              & 1               & SGD            & 0.001                                                            & 0.001                                                                               \\
            Cassette\_player  & 40              & 1               & Adam               & 0.005                                                            & 0.001                                                                               \\
            Chain\_saw        & 20              & 1               & SGD            & 0.01                                                             & 0.001                                                                               \\
            Church            & 40              & 1               & Adam               & 0.01                                                             & 0.001                                                                               \\
            French\_horn      & 20              & 1               & SGD            & 0.05                                                             & 0.001                                                                               \\
            Garbage\_truck    & 30              & 1               & Adam               & 0.005                                                            & 0.001                                                                               \\
            Gas\_pump         & 30              & 1               & Adam               & 0.01                                                             & 0.001                                                                               \\
            Golf\_ball        & 30              & 1               & SGD            & 0.01                                                             & 0.001                                                                               \\
            Parachute         & 12              & 1               & Adam               & 0.001                                                            & 0.001                                                                               \\ \hline
    \end{tabular}}
    \label{tab:params_imagenet}
\end{table}

\begin{table}[t]
    \centering
    \caption{Hyperparameters: Timeseries Experiments}
    {\scriptsize
        \begin{tabular}{llllll}
            \hline
            \textbf{Dataset} & \textbf{Radius} & \textbf{$\mu$} & \textbf{Optimizer} & \textbf{\begin{tabular}[c]{@{}l@{}}Learning\\ Rate\end{tabular}} & \textbf{\begin{tabular}[c]{@{}l@{}}Adversarial\\ Ascent \\ Step Size\end{tabular}} \\ \hline
            Epilepsy                                                                    & 10              & 0.5               & Adam               & $10^{-5}$                                                              & 0.1                                                                                \\
            \begin{tabular}[c]{@{}l@{}}Audio\\ Commands\end{tabular}                                                                   & 16              & 1.0             & Adam               & $10^{-3}$                                                              & 0.1                                                                                \\ \hline
        \end{tabular}
    }
    \label{tab:params_timeseries}
\end{table}

\begin{table}[t]
    \centering
    \caption{Hyperparameters: LFOC Experiments}
    {\scriptsize
        \begin{tabular}{llllll}
            \hline
            \textbf{Keyword} & \textbf{Radius} & \textbf{$\mu$} & \textbf{Optimizer} & \textbf{\begin{tabular}[c]{@{}l@{}}Learning\\ Rate\end{tabular}} & \textbf{\begin{tabular}[c]{@{}l@{}}Adversarial \\ Ascent \\ Step Size\end{tabular}} \\ \hline
            Marvin                                                                    & 32              & 1               & Adam                  & 0.001                                                           & 0.01                                                                                 \\
            Seven                                                                    & 36              & 1            & Adam                  & 0.001                                                           & 0.01                                                                                 \\
            Forward                                                                    & 40              & 1               & Adam                  & 0.001                                                           & 0.01                                                                                 \\
            
            Follow                                                                   & 20             & 1            & Adam                  & 0.0001                                                           & 0.01                                                                                     \\ \hline
    \end{tabular}}
    \label{tab:params_lfoc}
\end{table}

\section{LFOC Supplementary Experiments}
In Section~\ref{subsec:lfoc_res}, we compared \drocclf with various baselines for the \lfoc task where the goal is to learn a classifier that is accurate for both the positive class and the arbitrary OOD negatives. Figure~\ref{fig:lfoc_sup} compares the recall obtained by different methods on 2 keywords "Forward" and "Follow" with 2 different FPR. Table~\ref{tab:synth_keywords} lists the close negatives which were synthesized for each of the keywords. 

\section{Ablation Study}
\label{app:ablation}
\subsection{Hyper-Parameters}
Here we analyze the effect of two important hyper-parameters --- radius $r$ of the ball outside, which we sample negative points (set $\negset{i}(r)$), and $\mu$ which is the weightage given to the loss from adversarially generated negative points (See Equation~\ref{eq:drocc-loss}). We set $\gamma = 1$ and hence recall that the negative points are sampled to be at a distance of r from the positive points. 

Figure~\ref{fig:abl_r_airplane}, \ref{fig:abl_r_deer} and \ref{fig:abl_r_truck} show the performance of \drocc with varied values of $r$ on the CIFAR-10 dataset. The graphs demonstrate that sampling negative points quite far from the manifold (setting $r$ to be very large), causes a drop in the accuracy since now \drocc would be covering the normal data manifold loosely causing high false positives. At the other extreme, if the radius is set too small, the decision boundary could be too close to the positive and hence lead to overfitting and difficulty in training the neural network. Hence, setting an appropriate radius value is very critical for the good performance of \drocc.

Figure~\ref{fig:abl_lamda_airplane}, \ref{fig:abl_lamda_deer} and \ref{fig:abl_lamda_truck} show the effect of $\mu$ on the performance of \drocc on CIFAR-10.

\subsection{Importance of gradient ascent-descent technique}
In the Section~\ref{sec:drocc_formulation} we formulated the \drocc's optimization objective as a saddle point problem (Equation~\ref{eq:drocc-loss}). We adopted the standard gradient descent-ascent technique to solve the problem replacing the $\ell_p$ ball with $\negset{i}(r)$.
Here, we present an analysis of DROCC without the gradient ascent part i.e., we now sample points at random in the set of negatives $\negset{i}(r)$. We call this formulation as \droccrand. Table~\ref{tab:cifar_random}  shows the drop in performance when negative points are sampled randomly on the CIFAR-10, hence emphasizing the importance of gradient ascent-descent technique. Since $\negset{i}(r)$ is high dimensional, random sampling does not find points close enough to manifold of positive points.


\section{Experiment details and Hyper-Parameters for Reproducibility}
\subsection{Tabular Datasets}
Following previous work, we use a base network consisting of a single fully-connected layer with 128 units for the deep learning baselines. For the classical algorithms, the features are input to the model. Table~\ref{tab:params_tabular} lists all the hyper-parameters for reproducibility.

\subsection{CIFAR-10}
DeepSVDD uses the representations learnt in the penultimate layer of LeNet \cite{lecun1998gradient} for minimizing their one-class objective. To make a fair comparison, we use the same base architecture. However, since \drocc formulates the problem as a binary classification task, we add a final fully connected layer over the learned representations to get the binary classification scores. Table~\ref{tab:params_cifar} lists the hyper-parameters which were used to run the experiments on the standard test split of CIFAR-10.

\subsection{ImageNet-10}
MobileNetv2 \cite{mobilenetv2} was used as the base architecture for DeepSVDD and \drocc. Again we use the representations from the penultimate layer of MobileNetv2 for optimizing the one-class objective of DeepSVDD. The width multiplier for MobileNetv2 was set to be 1.0. Table~\ref{tab:params_imagenet} lists all the hyper-parameters.

\subsection{Time Series Datasets}
To keep the focus only on comparing \drocc against the baseline formulations for OOD detection, we use a single layer LSTM for all the experiments on Epileptic Seizure Detection, and the Audio Commands dataset. The hidden state from the last time step is used for optimizing the one class objective of DeepSVDD. For \drocc we add a fully connected layer over the last hidden state to get the binary classification scores. Table~\ref{tab:params_timeseries} lists all the hyper-parameters for reproducibility. 

\subsection{LFOC Experiments on Audio Commands}
For the Low-FPR classification task, we use keywords from the Audio Commands dataset along with some synthesized near-negatives. The training set consists of 1000 examples of the keyword and 2000 randomly sampled examples from the remaining classes in the dataset. The validation and test set consist of 600 examples of the keyword, the same number of words from other classes of Audio Commands dataset and an extra synthesized 600 examples of close negatives of the keyword (see Table~\ref{tab:synth_keywords}) 
A single layer LSTM, along with a fully connected layer on top on the hidden state at last time step was used. Similar to experiments with DeepSVDD, DeepSAD uses the hidden state of the final timestep as the representation in the one-class objective. An important aspect of training DeepSAD is the pretraining of the network as the encoder in an autoencoder. We also tuned this pretraining to ensure the best results.

\end{document}